\documentclass[twoside]{article}

\usepackage[accepted]{aistats2026}
\newcommand{\vx}{\mathbf{x}}
\newcommand{\vy}{\mathbf{y}}

\newcommand{\mA}{\mathbf{A}}
\newcommand{\mB}{\mathbf{B}}
\newcommand{\mC}{\mathbf{C}}

\newcommand{\mF}{\mathbf{F}}
\newcommand{\mG}{\mathbf{G}}
\newcommand{\mH}{\mathbf{H}}
\newcommand{\mI}{\mathbf{I}}

\newcommand{\mK}{\mathbf{K}}

\newcommand{\mO}{\mathbf{O}}

\newcommand{\mQ}{\mathbf{Q}}
\newcommand{\mR}{\mathbf{R}}

\newcommand{\mU}{\mathbf{U}}
\newcommand{\mV}{\mathbf{V}}
\newcommand{\mW}{\mathbf{W}}
\newcommand{\mX}{\mathbf{X}}

\newcommand{\mGamma}{\mathbf{\Gamma}}

\newcommand{\mLambda}{\mathbf{\Lambda}}

\newcommand{\mSigma}{\mathbf{\Sigma}}

\newcommand{\vdelta}{\boldsymbol{\delta}}

\newcommand{\vtheta}{\boldsymbol{\theta}}

\newcommand{\vmu}{\boldsymbol{\mu}}

\newcommand{\vsigma}{\boldsymbol{\sigma}}

\newcommand{\vpsi}{\boldsymbol{\psi}}

\newcommand{\sE}{\mathbb{E}} % Expectation

\newcommand{\sR}{\mathbb{R}} % Real Number

\newcommand{\mathD}{\mathcal{D}} % Dataset

\newcommand{\mathL}{\mathcal{L}} % Loss FUnction

\newcommand{\mathN}{\mathcal{N}}
\newcommand{\mathO}{\mathcal{O}} % Order

\usepackage{amsmath}
\usepackage{amssymb}
\usepackage{mathtools}
\usepackage{amsthm}
\usepackage{algorithm}
\usepackage{algorithmic}
\usepackage{listings}
\usepackage{svg}
\usepackage{graphicx}
\usepackage{multirow}
\usepackage{float}
\usepackage{mathrsfs}

\usepackage{hyperref}
\usepackage{url}
\usepackage{bbm}
\usepackage{colortbl}
\usepackage{booktabs}
\usepackage{paralist}
\usepackage{soul}
\usepackage{xcolor}
\usepackage{natbib}
\newtheorem{theorem}{Theorem}[section]

\newtheorem{lemma}[theorem]{Lemma}
\newtheorem{corollary}[theorem]{Corollary}

\newtheorem{remark}[theorem]{Remark}
\newtheorem{assumption}[theorem]{Assumption}

\newcommand\norm[1]{\left\lVert#1\right\rVert}

\begin{document}



\twocolumn[

% \aistatstitle{On the Gradient Regularization of Natural Gradient Descent: Frequentist and Bayesian Approaches}
\aistatstitle{Gradient Regularized Natural Gradients}

% \aistatsauthor{Satya Prakash Dash$^{*,1}$ \And Hossein Abdi$^{*,1}$ \And Wei Pan$^{1}$ \And Samuel Kaski$^{1,2}$ \And Mingfei Sun$^{1}$}

\aistatsauthor{
  \makebox[\textwidth][c]{
    Satya Prakash Dash$^{1,*}$ \hspace{0.6em} 
    Hossein Abdi$^{1,*}$ \hspace{0.6em} 
    Wei Pan$^{1}$ \hspace{0.6em} 
    Samuel Kaski$^{1,2}$ \hspace{0.6em} 
    Mingfei Sun$^{1}$
  }
}

\aistatsaddress{ $^{1}$ Department of Computer Science, The University of Manchester, United Kingdom
\\
$^{2}$ Department of Computer Science, Aalto University, Finland
\\
\{satyaprakash.dash, hossein.abdi, wei.pan, samuel.kaski, mingfei.sun\}@manchester.ac.uk
} ]

\begingroup
\renewcommand\thefootnote{}
\footnotetext{$^{*}$ Equal contribution.}
\endgroup

\begin{abstract}

Gradient regularization (GR) has been shown to improve the generalizability of trained models. While Natural Gradient Descent has been shown to accelerate optimization in the initial phase of training, little attention has been paid to how the training dynamics of second-order optimizers can benefit from GR. In this work, we propose Gradient-Regularized Natural Gradients (GRNG), a family of scalable second-order optimizers that integrate explicit gradient regularization with natural gradient updates. Our framework introduces two frequentist algorithms: Regularized Explicit Natural Gradient (RENG), which utilizes double backpropagation to explicitly minimize the gradient norm, and Regularized Implicit Natural Gradient (RING), which incorporates regularization implicitly into the update direction. We also propose a Bayesian variant based on a Regularized-Kalman formulation that eliminates the need for FIM inversion entirely. We establish convergence guarantees for GRNG, showing that gradient regularization improves stability and enables convergence to global minima. Empirically, we demonstrate that GRNG consistently enhances both optimization speed and generalization compared to first-order methods (SGD, AdamW) and second-order baselines (K-FAC, Sophia), with strong results on vision and language benchmarks. 

% Our findings highlight gradient regularization as a principled and practical tool to unlock the robustness of natural gradient methods for large-scale deep learning.

% Our framework provides two complementary algorithms: a frequentist variant that avoids explicit inversion of the Fisher Information Matrix (FIM) via structured approximations, and a Bayesian variant

% Natural Gradient Descent (NGD) leverages the local curvature information of the non-convex loss landscape by the Fisher Information Matrix (FIM) to adjust the update step size accordingly. 
% However, its applicability remains challenging because the Fisher matrix computed at a given iteration can be rank-deficient, that is, some of its eigenvalues may be zero.
% In this study, we propose novel gradient regularized optimization algorithms for both Frequentist and Bayesian frameworks that enable efficient approximation of the regularized natural gradients.
% Our approach avoids explicitly inverting the FIM in the Frequentist setting and eliminates the need for it in the Bayesian setting by employing the Regularized-Kalman algorithm. Despite their distinct formulations, both methods ultimately yield parameter updates that closely approximate the natural gradient direction. We benchmark our algorithms across image and language datasets and demonstrate that Regularized-NGD performs comparably or better than Adam and other baselines in high-data regimes, while Regularized-Kalman achieves similar or superior performance in low-data regimes.

\end{abstract}

% \vspace{-0.5cm}
\section{Introduction}

% 1. general background on optimization
% 2. GR and NGD in general
% 3. combination of NGD and GR
% 4. contribution

The rapid advancement of deep neural networks (DNNs) in recent years has driven significant progress across a wide range of domains, including computer vision and natural language processing. Despite these successes, the design of optimization algorithms that simultaneously ensure fast convergence and improved generalization in the training of these large-scale models remains a fundamental challenge \citep{bottou2018optimization}.

Recently, Natural Gradient Descent (NGD) methods have attracted increasing attention in the context of training DNNs. This interest is largely motivated by their favorable theoretical properties—most notably, their ability to identify optimal solutions in as few as $\mathcal{O}(1)$ iterations under certain conditions \citet{zhang2019fast}. In addition, NGD methods are well-suited for addressing ill-conditioned optimization landscapes, where they can provide significant advantages over traditional approaches by accelerating convergence and improving training stability \citet{becker1988improving, martens2020new}.
% Recently, there has been a surge of interest in using second-order optimizers for optimizing Deep Neural Networks (DNN) due to their good theoretical properties of solving a problem in $\mathcal{O}(1)$ step \cite{zhang2019fast} and their applicability in solving ill-conditioned problems and speed up the convergence of training \cite{becker1988improving}, \cite{martens2020new}.

% In this context, parameter regularization techniques\mingfei{gradient regularization is a special type of parameter reg.? what's the point of mentioning parameter regularization?} —most notably Tikhonov regularization implemented through the Levenberg–Marquardt scheme \cite{more2006levenberg}—have demonstrated strong empirical performance during the early stages of training. Such methods have been shown to accelerate convergence in the initial optimization phase \cite{martens2015optimizing} and have additionally been associated with improved generalization of the trained models \cite{jastrzebski2021catastrophic, jia2020information}.

% In this situation, parameter regularization techniques, such as Tikhonov regularization with the Levenberg-Marquardt scheme \cite{more2006levenberg} have proven to work well in the initial phase of training. They have proven to speed up the initial phase of training \cite{martens2015optimizing} and have also been linked to increased generalizability of the trained models \cite{jastrzebski2021catastrophic}, \cite{jia2020information}.

Gradient Regularization (GR) has been extensively investigated in the context of gradient descent \citet{smith2021on, barrett2020igr}. GR-based methods have been shown to accelerate convergence in the initial phase \citet{jastrzebski2021catastrophic} and have additionally been associated with improved generalization of the trained models \citet{smith2021origin}.
% Meanwhile, implicit and explicit forms of Gradient Regularization (GR) have been extensively investigated in the context of gradient descent dynamics \cite{smith2021on, barrett2020igr}.
Theoretically, it has been shown that employing higher learning rates induces an implicit GR effect. Complementarily, explicit regularization of the objective function via the gradient norm has been empirically demonstrated to bias the optimization trajectory toward finding flatter regions of the loss landscape \citet{drucker1992improving} \citet{barrett2020igr}. This, in turn, has been linked to enhanced generalization performance and improved robustness of trained models.
% Implicit and explicit Gradient Regularization (GR) have been extensively studied for gradient descent dynamics \cite{smith2021on}, \cite{barrett2020igr}, which theoretically shows that with higher learning rates the gradient descent dynamics implicitly penalizes the second moment of the loss gradient and explicitly regularizing the loss function with norm of the gradient has empirically demonstrated that this surrogate loss function steers the optimization trajectory to reach the areas of loss landscape with shallower or flatter regions \cite{drucker1992improving}, hence improving the generalizablity and robustness of the trained models.

Despite the extensive body of work on gradient regularization in the context of gradient descent, relatively little attention has been devoted to (i) understanding its convergence properties when combined with natural gradient optimization methods, and (ii) developing empirical techniques to assess its impact on the generalization performance of NGD optimizers. A more extensive review of related work is provided in \autoref{Related work}.

\begin{figure}[h]
    \includegraphics[width=0.45\textwidth]{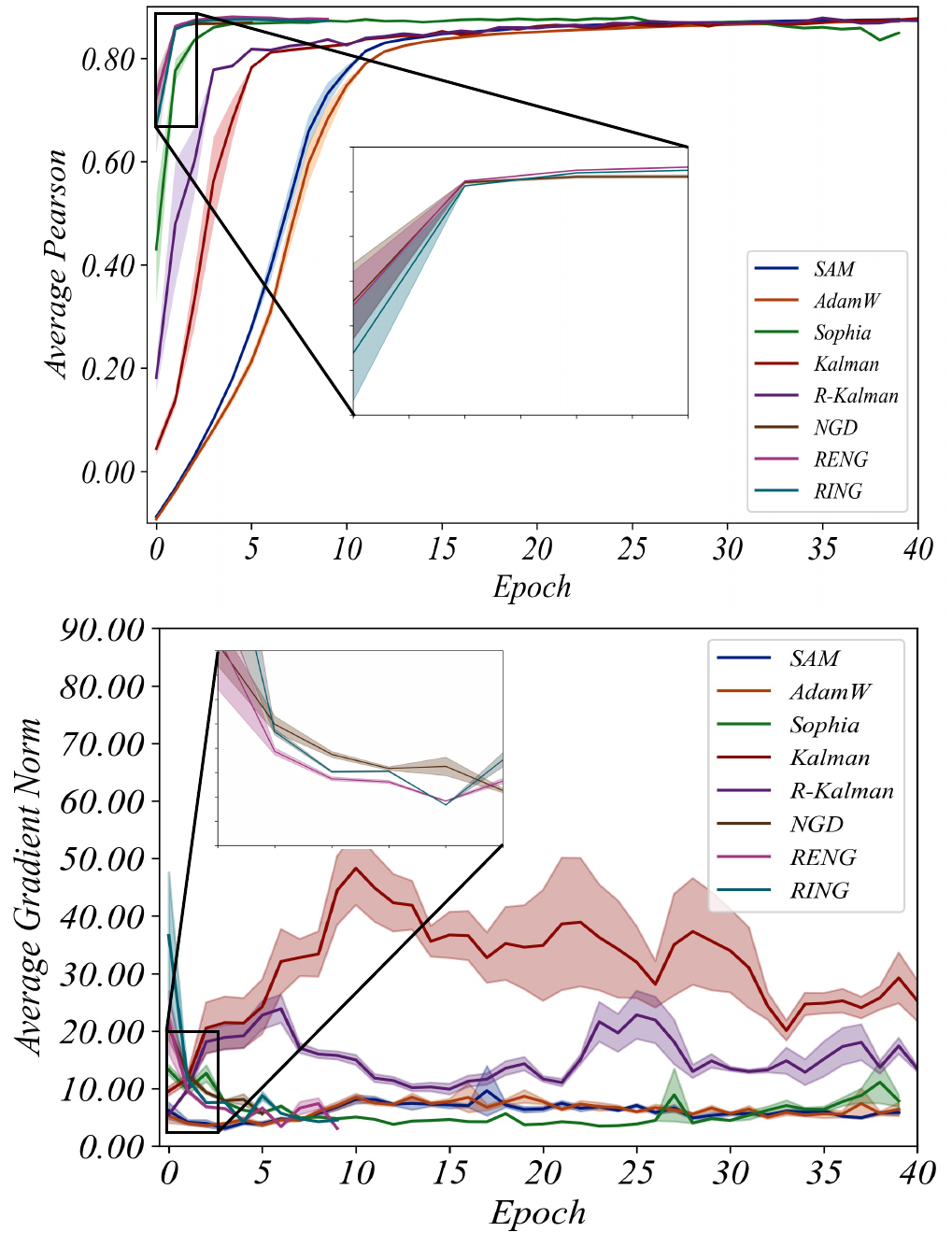}
    \caption{Average Gradient Norm (left) and Average Pearson Correlation (right) for SAM, AdamW, Sophia, Kalman, R-Kalman, RING, and RENG on the STS-B dataset.}
    % \vspace{-0.5cm}
    \label{fig:intro_fig}
\end{figure}

% Although these results have been extensively studied in gradient descent dynamics, there has not been much research on (1) the convergence properties of GR with second order-methods and (2) empirical techniques which show the generalizablility of GR with second-order optimizers like Natural Gradient Descent (NGD).

% \vspace{-0.2cm}

In this work, we investigate the convergence properties of Gradient-Regularized Natural Gradients (GRNG). We provide a theoretical analysis showing that, for a simple two-layer neural network, GRNG converges to the global minimum. Beyond theory, we empirically evaluate the generalization performance of GRNG in comparison with a range of established optimizers, including first-order methods such as Stochastic Gradient Descent (SGD) and Sharpness-Aware Minimization (SAM) \citet{foret2020sharpness}, quasi-second-order methods such as Adam \citet{kingma2014adam}, and second-order baselines without gradient regularization, such as K-FAC with Tikhonov regularization \citet{martens2015optimizing}, and Sophia \citep{liu2023sophia}. Our framework further introduces two frequentist algorithms (see Algorithm \ref{alg:ring-reng}): RENG and RING. RENG employs double backpropagation to explicitly regularize the objective, whereas RING avoids this overhead by estimating an update direction that reflects the regularization implicitly. These are complemented by a Bayesian variant based on a Regularized-Kalman formulation that eliminates the need for FIM inversion entirely. Together, these advances make GRNG scalable and applicable to large-scale architectures, including Transformers in both language and vision domains.

% Our framework further introduces two complementary scalable algorithms: a frequentist variant that employs block-diagonal Kronecker-factored approximations of the Fisher Information Matrix (FIM) together with Newton's Iteration \cite{shamanskii1967modification} and the lazy-Hessian technique \cite{doikov-lazyhessian}, and a

% In this paper, we aim at understanding the convergence properties of gradient regularized NGD (GRNG) and \textit{theoretically prove} that for a simple two layer neural network GRNG converges to global minima. We \textit{empirically validate} the generalizability of this algorithm with the first-order counterparts like SGD, sharpness-aware minimization (SAM) \cite{foret2020sharpness}, quasi-second-order optimizers like Adam \cite{kingma2014adam} and second-order counterpart without gradient regularization like K-FAC with Tikhonov regularization \cite{martens2015optimizing}. Furthermore, we provide \textit{scalable algorithm} which uses block diagonal Kronecker factorization approximation of FIM and uses Newton's Iteration \cite{shamanskii1967modification} and lazy-Hessian technique \cite{doikov-lazyhessian} to make our algorithm scalable for application to Transformers in language as well as in vision domain.

In \ref{fig:intro_fig}, we illustrate the advantages of incorporating gradient regularization. SAM, a gradient-regularized first-order optimizer, is able to locate flatter minima more efficiently than a quasi-second-order method such as AdamW, highlighting the effectiveness of gradient regularization. Furthermore, our proposed GRNG methods (Algorithms \ref{alg:ring-reng} and \ref{alg:kalman}) achieve a rapid reduction in gradient norm compared to their unregularized counterparts, and converges within a few epochs of training.

% In the \autoref{fig:intro_fig}, we show the advantage of using gradient regularization. First, we observe that a first-order gradient regularized objective like SAM finds flatter minima faster than a quasi second order method like AdamW, which shows the benefits of using gradient regularization. Secondly, we observe that RING \& RENG (Algorithm \autoref{alg:ring-reng}) quickly reduces the gradient norm as compared to Adam and Sophia and converges within a few epochs. Finally we also show the regularized Kalman update (Algorithm \autoref{alg:kalman}) reduces gradient norm by a significant margin than the vanilla Kalman update. 

\section{Problem Statement}

Formally, suppose we are given a dataset $\mathD = \{(\vx_k, \vy_k) \sim p^*(\vx,\vy) \}_{k=1}^{N}$, where $\vx_k \in \sR^{d_i}$ and $\vy_k \in \sR^{d_o}$ denote the input and output vectors, respectively; and are sampled i.i.d. from the true data distribution $p^\ast(\vx,\vy)$. 
Consider a (possibly pre-trained) model defined by $\hat{\vy} = h(\vx, \vtheta)$, which is parameterized by $\vtheta \in \sR^n$. 
Given an input $\vx$, the model defines a predictive distribution $p(\vy|\vx,\vtheta)$, from which the output $\vy$ can be sampled, and $\hat{\vy}$ is a representative statistic of this distribution.
The objective of prediction in machine learning is to determine model parameters that make the model's predicted distribution $p(\vy|\vx, \vtheta)$ closely approximate the true conditional data distribution $p^*(\vy|\vx)$. Specifically, prediction aims to learn a mapping from inputs $\vx$ to outputs $\vy$, in order to predict the outputs for new and unseen inputs.
This mapping is learned by training or fine-tuning the model $\hat{\vy} = h(\vx, \vtheta)$. In the following, we introduce that both frequentist and Bayesian optimization frameworks can be employed in this context.

% The objective of prediction in machine learning is to learn a mapping from inputs $\vx$ to outputs $\vy$, in order to predict the outputs for new and unseen inputs.
% This mapping is learned by training or fine-tuning the model $\hat{\vy} = h(\vx, \vtheta)$.
% Specifically, prediction aims to determine model parameters that make the model's predicted distribution $p(\vy|\vx, \vtheta)$ closely approximate the true conditional data distribution $p^*(\vy|\vx)$.
% In the following, we introduce that both frequentist and Bayesian optimization frameworks can be employed in this context.

\paragraph{Frequentist Approach.}
In Frequentist approaches, this objective is typically achieved by minimizing the Kullback–Leibler (KL) divergence between the true conditional distribution and the model distribution, $D_{KL}(p^*(\vy|\vx)||p(\vy|\vx, \vtheta))$, which is equivalent to minimizing the negative log-likelihood loss: $\mathcal{L}(\vtheta) = -\sE_{p^\ast(\vx,\vy)}[\ln p(\vy|\vx,\vtheta)]$ \citep{goodfellow2016deep, murphy2012machine, bishop2006pattern}.
When the update step ($\vdelta$) is infinitesimally small, its second-order approximation is given as
% 
% \vspace{-0.8cm}
% 
\begin{equation}
  \label{eq:lossdelta}
  \mathcal{L}(\vtheta + \vdelta) =  \mathcal{L}(\vtheta) + \nabla_{\vtheta} \mathcal{L}(\vtheta)^\top \vdelta + \frac{1}{2} \vdelta^{\top} \mF(\vtheta) \vdelta + \mathO(|\vdelta|^3), 
\end{equation} 
where $\vdelta$ is the parameter update and the curvature of the loss function is given by the Fisher information matrix $\mF(\vtheta)$~\citep{pascanu2013revisiting}, defined as follows:
% 
% \vspace{-0.5cm}
% 
\begin{equation}
  \label{eq:fisher}
  \mF(\vtheta) = \sE_{p(\vy|\vx,\vtheta)}[ \nabla_{\vtheta} \ln p(\vy|\vx,\vtheta) \cdot \nabla_{\vtheta} \ln p(\vy|\vx,\vtheta)^{\top} ].
\end{equation}
% Fisher information can be viewed as the covariance of the \emph{score} -- the gradient of the log-likelihood loss $\nabla_{\vtheta} \ln p(\vy|\vx,\vtheta)$ -- under the \emph{model output distribution}. 
Solving the quadratic function of \autoref{eq:lossdelta} in the vicinity of $\vtheta$ gives the natural gradient direction: $\vdelta^{\ast} = -\mF(\vtheta)^{-1} \nabla_{\vtheta} \mathcal{L}(\vtheta)$, where the gradient of the loss function is with respect to the \emph{data distribution} (also called batch-gradients).
Note that Fisher information matrix takes the expectation over the \emph{model distributions}.

% \vspace{-0.2cm}

\paragraph{Bayesian Approach.}

From a Bayesian perspective, the parameter $\vtheta$ is treated as a random variable. 
Given a data point $(\vx_k,\vy_k)$ in a stream of data, we aim to recursively estimate the posterior distribution of the parameters using Bayesian inference \citep{murphy2023probabilistic, murphy2012machine}:
\begin{equation} \label{Eq. posterior}
    p(\vtheta_k \mid \mathD_{1:k}) \propto p(\vy_k \mid \vx_k , \vtheta_{k-1}) p(\vtheta_{k-1} \mid \mathD_{1:k-1}),
\end{equation}
where $\mathD_{1:k} \subseteq \mathD$ indicates the subset of dataset used for training until $k^{th}$ data. 
And $p(\vy_k \mid \vx_k , \vtheta_{k-1})$ represents the likelihood function, $p(\vtheta_{k-1} \mid \mathD_{1:k-1})$ and $p(\vtheta_k \mid \mathD_{1:k})$ denote the prior and posterior, respectively.
By leveraging a Kalman-based approach, we recursively update and compute the posterior distribution over the model parameters.

% Bayesian approach treats $\vtheta$ as a random variable and infers its posterior distribution given the data $\mathD$ by applying Bayes' rule \citep{murphy2023probabilistic, murphy2012machine}.

% \begin{equation}
%   \label{eq:bayes}
%   p(\vtheta \mid \mathD) = \frac{p(\mathD \mid \vtheta) p(\vtheta)}{p(\mathD)} = \frac{p(\mathD \mid \vtheta) p(\vtheta)}{ \int p(\mathD \mid \vtheta') p(\vtheta') d\vtheta' },
% \end{equation}

% where $p(\mathD \mid \vtheta)$ is the likelihood, $p(\vtheta)$ represents the prior, and $p(\mathD)$ denotes the evidence (or marginal likelihood).
% Having the posterior $p(\vtheta \mid \mathD)$, one can make predictions for a new input by: 
% \begin{equation}
%   \label{eq:bayes_prediction}
%   p(\vy' \mid \vx', \mathD) = \int p(\vy' \mid \vx', \vtheta) \, p(\vtheta \mid \mathD) d\vtheta
%   % = \sE_{p(\vtheta \mid \mathD)}[p(\vy' \mid \vx', \vtheta)] .
% \end{equation}

% Computing the posterior distribution via Bayes' theorem is referred to as exact Bayesian inference. 
% This process is generally computationally intractable in high-dimensional models, as it requires evaluating an integral over the entire parameter space $\vtheta$ to obtain the posterior and predictive distributions. Consequently, approximation techniques are typically employed to estimate the posterior and enable practical Bayesian inference.

% \section{Related work}

% \section{Preliminaries}
% \label{sec:Preliminaries and Background}

% \vspace{-0.2cm}

\section{Methods}
Direct inversion of the full Fisher Information Matrix is computationally intractable in large-scale settings. To address this challenge, we adopt two strategies: in the frequentist approach, we employ block-diagonal Kronecker-factored approximations of the FIM, combined with Newton’s Iteration and the lazy-Fisher technique; in the Bayesian approach, we bypass the need for FIM inversion entirely by leveraging a Kalman-based technique.

\subsection{Frequentist: Inverting Fisher Layerwise and Lazily}

\paragraph{Kronecker-Factored Approximation.}
We first take the frequentist approach to address the computational challenge.
We develop our algorithm around the layer-wise structure of the model's weight matrices. For simplicity, we denote the weight matrix of layer $i$ as $\mW_i \in \sR^{\omega_i \times \omega_i^{'}}$, where $\omega_i$ and $\omega_i'$ represent the input and output dimensions of the layer, respectively. The full parameter vector $\vtheta$ is then defined as the concatenation of the vectorized layer weights: 
$\vtheta = \text{vec}(\mW_1, \mW_2, ... , \mW_L)$, where $L$ denotes the number of layers in the model. We also adopt $\boldsymbol{x}_i$ as input to layer $i$, and $\boldsymbol{e}_i$ as backpropagated error until that layer.
The loss function can be expressed in terms of $\boldsymbol{x}_{i-1}$ and $\boldsymbol{e}_i$, which are the canonical basis of the loss: 
% 
% \vspace{-0.8cm}
% 
\begin{small}
\begin{equation}
  \label{eq:reparam}
  \mathcal{L}(\vtheta) \equiv \mathcal{L} (\mW_1, \mW_2, ... , \mW_L) \equiv \mathcal{L}(\boldsymbol{x}_{0}, ..., \boldsymbol{x}_{L-1}, \boldsymbol{e}_1, ..., \boldsymbol{e}_L)
\end{equation}
\end{small}
The gradient for each weight matrix can be rewritten on a suitable canonical basis as:
% 
% \vspace{-0.5cm}
% 
\begin{equation}
  \label{eq:batchgrad}
  \nabla_{\mW_i} \mathcal{L}(\mW_i) = \sE_{p^*(\vx,\vy)}[ \boldsymbol{e}_i \boldsymbol{x}_{i-1}^{\top} ] = \sE_{p^*(\vx,\vy)} [\boldsymbol{x}_{i-1} \otimes \boldsymbol{e}_i] ,
\end{equation} 
where $\otimes$ denotes the Kronecker product. 
Also, inspired by \citep{exactngd}, we approximate the full Fisher Information Matrix as follows: 
\begin{align}
  \label{eq:fisherij}
  \begin{split}
      \mF(\vtheta) & = \sE_{p(\vy|\vx,\vtheta)}[ \boldsymbol{x} \boldsymbol{x}^{\top} \otimes \boldsymbol{e} \boldsymbol{e}^{\top} ] \\
       & \approx \sE_{p(\vy|\vx,\vtheta)}[ \boldsymbol{x} \boldsymbol{x}^{\top} ] \otimes \sE_{p(\vy|\vx,\vtheta)}[ \boldsymbol{e} \boldsymbol{e}^{\top} ]
  \end{split}
\end{align}
Assuming each weight matrix is independent of the others, for each model layer we get: 
\begin{equation}
  \label{eq:fisherii}
  \mF_{ii}(\mW_i) = \sE_{p(\vy|\vx,\vtheta)}[ \boldsymbol{x}_{i-1} \boldsymbol{x}_{i-1}^{\top} ] \otimes \sE_{p(\vy|\vx,\vtheta)}[ \boldsymbol{e}_i \boldsymbol{e}_i^{\top} ]
\end{equation} 
where we refer to $\mLambda_{i-1} = \sE_{p(\vy|\vx,\vtheta)}[ \boldsymbol{x}_{i-1} \boldsymbol{x}_{i-1}^{\top} ]$ as the activation matrix, and $\mGamma_i = \sE_{p(\vy|\vx,\vtheta)}[ \boldsymbol{e}_i \boldsymbol{e}_i^{\top} ]$ as the error matrix.
Then, we can write the weight update for layer $\mW_i$ as: 
\begin{align}
\label{eq:weightupdate}
\begin{split}
    \Delta \mW_i & \propto \mF_{ii}(\mW_i)^{-1} \nabla_{\mW_i} \mathcal{L}(\mW_i) \\
     &\propto \mLambda_{i-1}^{-1} \cdot \sE_{p^*(\vx,\vy)}[ \boldsymbol{e}_i \boldsymbol{x}_{i-1}^{\top} ] \cdot \mGamma_i^{-1}
\end{split}
\end{align}
To compute this equation, we use the mixed Kronecker matrix-vector product property: $(\mA \otimes \mB) \text{vec}(\mC) = \text{vec}(\mB\mC\mA^{\top})$.
% Note that this formulation of Fisher Information incorporates curvature information for all the parameters in each weight block, which is more informative than the diagonal approximation taken in Adam and its variants~\citep{loshchilov2017decoupled,pmlr-v80-shazeer18a-adafacor,liu2019variance-radam,liu2023sophia}.
% 
\paragraph{Lazy Fisher.}Computing the Fisher information matrix at every training step is computationally expensive. 
We take advantage of the idea of using the same Fisher information matrix for the next few training steps, which is also known as \textit{Lazy Hessian} in Hessian estimation\citep{shamanskii1967modification}.
It has been shown to achieve global convergence in damped Newton methods\citep{wang2006further} and proximal Newton-type algorithms \citep{adler2020new} in the context of convex optimization, and global convergence guarantees for its damped and regularized variants\citep{doikov-lazyhessian,chayti-doikov-2023unified}. 
We empirically investigate the effectiveness of the lazy Fisher approach in our experiments. 
This technique significantly reduces training time -- often by orders of magnitude -- compared to standard NGD, as it avoids inverting the Fisher matrix at every iteration. 

\subsection{Frequentist: Implicit and Explicit Gradient Regularized Natural Gradients}
\label{subsec:GRNG}

Although the weight update derived in \autoref{eq:weightupdate} provides a quadratic approximation of the loss function, its applicability remains challenging because the error matrix computed at a given iteration can be rank-deficient. Furthermore, as shown in \autoref{eq:lossdelta}, the loss function is approximated using a second-order Taylor expansion around the current solution. Hence, the approximation can be accurate in the neighborhood of the solution, particularly for nearly linear loss functions, which is a condition that typically does not hold in the early stages of training. Moreover, due to the non-convex nature of the optimization landscape, the weight update direction in \autoref{eq:weightupdate} might not yield an optimal descent direction. In such cases, an appropriate regularization is necessary, as the steepest descent may lead to suboptimal performance. From an empirical perspective, we also consider how to obtain the weight update for pointwise operations. Therefore, regularization methods can be a helpful tool to accelerate the training process.

\paragraph{Gradient Regularization.}
Gradient regularization adds the norm of gradient of the loss to encourage the search for flatter minima \citep{barrett2020igr}. Here, in this case, we will add a regularizer averaged over the model distribution, meaning that we regularize our objective to create a direction where the gradient averaged over the model distribution is minimized. 
Gradient regularization to the quadratic loss approximation can be defined as:
\begin{align}
  \label{eq:grad_reg_loss}
  & \mathcal{L}_G(\vtheta + \vdelta) =  \mathcal{L}(\vtheta) + \nabla_{\vtheta} \mathcal{L}(\vtheta) \vdelta \notag\\
  & + \frac{1}{2} \vdelta^{\top} \Big[ \mF(\vtheta) + \rho \norm{\sE_{p(\vy|\vx,\vtheta)}[\nabla_{\vtheta} \mathcal{L}(\vtheta)]}^2_2 \mI \Big] \vdelta + \mathO(|\vdelta|^3),
\end{align} 
where $\rho$ denotes the damping coefficient.
We consider two approaches to applying gradient regularization: \textit{explicit} (in the \textbf{RENG} algorithm) and \textit{implicit} (in the \textbf{RING} algorithm). The explicit method involves computing the \textit{L2-norm} of the primary loss gradient and then backpropagating through this norm—an approach known as \textit{double backpropagation} \citet{drucker1992improving}. In contrast, the implicit method avoids this extra step by directly estimating an update direction that reflects the regularization effect.

The double backpropagation technique has demonstrated its effectiveness in improving generalization by \citep{doublebp}. 
\citep{flatminima_sepp} has also shown that it encourages convergence toward flatter minima in the loss landscape.
 \citep{barrett2020igr} has theoretically studied this technique under vanilla SGD and proves that gradient regularization encourages taking a shallower path to reach local minima. 
% It has also been applied to improve GAN training in \citep{gan_training}.
For the weight update step in the RING and RENG algorithms, we adopt the following equations from \citep{exactngd}.
\begin{subequations} \label{eq:grad_weight_update}
    \begin{align}
        \Delta \mW_i = & \frac{\alpha}{L} \left( \mLambda_{i-1} + \sqrt{\rho} \norm{\mLambda_{i-1}}_2 \mI \right)^{-1} \notag\\
        & \cdot \sE_{p^*(\vx,\vy)}[ \boldsymbol{e}_i \boldsymbol{x}_{i-1}^{\top} ] \cdot \left( \mGamma_i + \sqrt{\rho} \norm{\mGamma_{i}}_2 \mI \right)^{-1}
        \label{Eq.RING} \\
        \Delta \mW_i = & \frac{\alpha}{L} \left( \mLambda_{i-1} + \sqrt{\rho}\mI \right)^{-1} \notag\\
        & \cdot \sE_{p^*(\vx,\vy)}[ \boldsymbol{e}_i \boldsymbol{x}_{i-1}^{\top} ] \cdot \left( \mGamma_i + \sqrt{\rho}\mI \right)^{-1} . 
        \label{Eq. RENG}
    \end{align}
\end{subequations}
Here, we define $\tilde{\mLambda} := \mLambda + \lambda \mI$, and $\tilde{\mGamma} := \mGamma +\lambda \mI$, where $\lambda = \sqrt{\rho} \norm{\cdot}_2$ is for RING, and $\lambda = \sqrt{\rho}$ is for RENG. These regularizations can be viewed as modifications of the geometry induced by the Fisher information matrix.
The pseudocode of these algorithms is provided in \autoref{Algorithm}, and their computational complexity is analyzed in Appendix~\autoref{Frequentist Computational Complexity}.

\paragraph{Updating Inverse Matrix.}
Despite the computational advantages that Lazy Fisher offers, we observe that it can occasionally lead to training instabilities. 
The primary reason is the use of the regularization mechanism. Specifically, when the damping coefficient ($\rho$) is updated at each iteration, it must remain greater than the minimum eigenvalue of the Fisher matrix at that iteration. Failure to satisfy this condition may result in inaccurate updates along the eigenvectors with the smallest eigenvalues. To mitigate this issue, we approximate the updated matrix inversion using matrix differential techniques.
% \begin{proposition}
    % \label{pro:inverse_update}
    % By defining $\tilde{\mLambda} = \left( \mLambda + \lambda \mI \right)$ as the \textit{regularized} activation matrix, and $\tilde{\mGamma} = \left( \mGamma + \lambda \mI \right)$ as the \textit{regularized} error matrix, where $\lambda = \sqrt{\rho} \norm{\cdot}_2$ or $\lambda = \sqrt{\rho}$, and employing \textit{Lazy Fisher} technique,
    To this aim, the regularized activation ($\tilde{\mLambda}$) and error ($\tilde{\mGamma}$) matrices are updated as follows:
    \begin{subequations}
    \label{eq:inverse_update}
        \begin{align}
         \tilde{\mLambda}_{i_{k+1}}^{-1} &= \tilde{\mLambda}_{i_{k}}^{-1} - d\lambda_k \tilde{\mLambda}_{i_{k}}^{-1} \tilde{\mLambda}_{i_{k}}^{-1} + \mathcal{O}(d\lambda_k^2) \\
        \tilde{\mGamma}_{i_{k+1}}^{-1} &= \tilde{\mGamma}_{i_{k}}^{-1} - d\lambda_k \tilde{\mGamma}_{i_{k}}^{-1} \tilde{\mGamma}_{i_{k}}^{-1} + \mathcal{O}(d\lambda_k^2)
        \end{align}
    \end{subequations}
% \end{proposition}
% \begin{proof}
    See Appendix \autoref{Proof:pro:inverse_update} for the detailed derivation.
% \end{proof}
% This proposition provides an efficient approximation for updating the inverse matrices without explicitly recomputing the matrix inverses at each step.
% 
\paragraph{Newton’s Iteration.}
\autoref{eq:inverse_update} requires the initial inverses $\tilde{\mLambda}_i^{-1}$ and $\tilde{\mGamma}_i^{-1}$ to be precomputed, which can itself be computationally expensive. To address this, we employ \textit{Newton’s Iteration} method \citep{schulz1933iterative} using a cubic-order Neumann series approximation. Newton's iteration start from a scaled version of the transpose of the matrix ($\mX_0 = \alpha \mA^{\top}$) and iteratively reduce the residue using a Neumann series approximation: $\mX_{k+1} = \mX_k (3 \mI - 3\mA \mX_k + (\mA \mX_k)^2)$. This series will converge to $\mA^{-1}$ provided proper choice of $\alpha$ like $\alpha = \frac{1}{Tr(\mA^{\top} \mA)}$ or $\alpha = \frac{1}{\norm{\mA}_2 \norm{\mA}_{\infty}}$~\citep{ben1965iterative, ben1966iterative}.

\subsection{Bayesian: Gradient Regularized Kalman}

Leveraging a Kalman-based framework, we recursively update the posterior distribution over model parameters along the natural gradient direction, which avoids the explicit inversion of the Fisher Information Matrix within a Bayesian setting.

% From a Bayesian perspective, the parameter $\vtheta$ is treated as a random variable. 
% Given a data point $(\vx_k,\vy_k)$ in a stream of data, we aim to estimate the posterior distribution of the parameters using Bayesian inference \citep{murphy2023probabilistic, murphy2012machine}:
% \begin{equation} \label{Eq. posterior}
%     p(\vtheta_k \mid \mathD_{1:k}) \propto p(\vy_k \mid \vx_k , \vtheta_{k-1}) p(\vtheta_{k-1} \mid \mathD_{1:k-1}),
% \end{equation}
% where $\mathD_{1:k} \subseteq \mathD$ indicates the subset of dataset used for training until $k^{th}$ data. 
% And $p(\vy_k \mid \vx_k , \vtheta_{k-1})$ represents the likelihood function.

\paragraph{Kalman Algorithm.}
Following the standard Kalman filtering assumptions, we adopt a Gaussian variational approximation for the posterior distribution over model parameters: $p(\vtheta_{k} \mid \mathD_{1:k}) = \mathN(\vmu_k,\mSigma_k)$. 
In this regard, we assume that an initial prior distribution of the trainable parameters is given by a Gaussian distribution $p(\vtheta_0) = \mathN(\vmu_0,\mSigma_0)$. 
\autoref{Eq. posterior} can be recursively computed for each data point in $\mathD$ using Kalman filtering, which corresponds to following the natural gradient direction (see Appendix \autoref{Proof:Kalman_NGD} for details). 
Consider a Gaussian likelihood function as $p(\vy_k \mid \vx_k , \vtheta_{k-1}) = \mathN(\vy_k \mid \hat{\vy}_k, \mR_k)$; or more broadly, an exponential family non-Gaussian likelihood function expressed as $p(\vy_k \mid \vx_k , \vtheta_{k-1}) = f_{\text{exp}} (T(\vy_k)|\hat{\vy}_k)$, where $f_{\text{exp}}$ denotes the exponential family function, and $T(\cdot)$ represents the sufficient statistics associated with the exponential family.
Here, $\vy_k$ denotes the true output, while $\hat{\vy}_k$ corresponds to the model's predicted output. 
In this context, the observation noise covariance matrix $\mR_k$ is defined as $\mR_k = \text{Cov}(\vy_k|\hat{\vy}_k)$ (for non-Gaussian: $\mR_k = \text{Cov}(T(\vy_k)|\hat{\vy}_k)$), which captures the covariance between $\vy_k$ and $\hat{\vy}_k$ (for non-Gaussian: $T(\vy_k)$ and $\hat{\vy}_k$) \citep{ollivier2018online}.

Based on the information form of the Kalman algorithm (information filtering), two fundamental steps should be taken to estimate the posterior distribution of the trainable parameters \cite{simon2006optimal}:

\emph{Prediction Step:} Under the first step of the Kalman algorithm, the prior is predicted based on the posterior from the previous iteration, following the Gaussian transition model:
\begin{equation} \label{Eq. prior}
    p(\vtheta_{k \mid k-1} \mid \vtheta_{k-1}) = \mathN(\vtheta_{k \mid k-1} \mid \vtheta_{k-1}, \mQ_{k}) ,
\end{equation}
where $\mQ_{k}$ denotes the process noise covariance matrix, and the subscript $k | k{-}1$ indicates the predicted value at time step $k$ conditioned on information available up to time $k{-}1$.
Under this evolution, the predicted prior distribution is:
% 
% \vspace{-0.3cm}
% 
\begin{subequations} \label{Eq. LoKA Prediction}
        \begin{align} 
            \vmu_{k|k-1} &= \vmu_{k-1}
            \label{Eq. LoKO Prediction State}\\
            \mSigma^{-1}_{k|k-1} &= \mSigma^{-1}_{k-1} + \mQ_{k}
            \label{Eq. LoKA Prediction Covariance}
        \end{align}
\end{subequations}
\emph{Updating Step:} In the second step of the Kalman filter, the predicted prior is updated to estimate the posterior as follows:
\begin{subequations} \label{Eq. Algorithm Updating}
    \begin{align} 
        \mK_k &= \mSigma_{k|k-1} \mH_k^{\top} \left(\mH_k \mSigma_{k|k-1} \mH_k^{\top} + \mR_k\right)^{-1}
        \label{Eq. Algorithm diag gain update}\\
        \vmu_{k} &= \vmu_{k|k-1} + \mK_k \left(\vy_k - h(\vx_{k}, \vmu_{k|k-1}) \right)
        \label{Eq. Algorithm Updating Process}\\
        \mSigma^{-1}_{k} &= \mSigma^{-1}_{k|k-1} - \mH^{\top}_k \mR^{-1}_k \mH_k
        \label{Eq. Algorithm diag cov update}
    \end{align}
\end{subequations}
where $\mK_k$ is the Kalman gain, and $\mH_k$ indicates the Jacobian matrix of the model $h(\vx, \vtheta)$ with respect to the parameters $\vtheta$ at the point of $(\vx_{k}, \vmu_{k|k-1})$. 
Note that $h(\vx_{k}, \vmu_{k|k-1})$ represents the predicted values for the regression tasks and the predicted probabilities for the classification tasks.

\paragraph{Regularized Kalman.}
In line with the discussions presented in \autoref{subsec:GRNG} regarding the frequentist perspective, the gradient (or equivalently, Jacobian) regularization can be interpreted as a modification of the Fisher information matrix of the form $\tilde{\mF}_k = \mF_k + \rho \mH_k^{\top} \mH_k$. This interpretation reveals that the gradient-regularized natural gradient step can be expressed as a modified Kalman update, thereby establishing a direct connection between regularized Kalman filtering and RING/RENG algorithms.

Recall that the Fisher information matrix associated with a negative log-likelihood loss $\mathcal{L}$ is given by (see \autoref{eq:fisher}): $\mF(\vtheta) = \sE_{p(\vy|\vx,\vtheta)}[ \nabla_{\vtheta} \mathcal{L} \cdot \nabla_{\vtheta} \mathcal{L}^{\top} ]$.
At time step $k$, the contribution of a single observation can be expressed as $\mF(\vtheta_k) = \nabla_{\vtheta} \mathcal{L}_k \cdot \nabla_{\vtheta} \mathcal{L}_k^{\top}$. Applying the chain rule, $\nabla_{\vtheta}\mathcal{L}_k = \nabla_{\hat{\vy}}\mathcal{L}_k \cdot \nabla_{\vtheta}\hat{\vy}_k$, we obtain $\mF(\vtheta_k) = (\nabla_{\hat{\vy}}\mathcal{L}_k \cdot \nabla_{\vtheta}\hat{\vy}_k) \cdot (\nabla_{\hat{\vy}}\mathcal{L}_k \cdot \nabla_{\vtheta}\hat{\vy}_k)^{\top}$. By matrix multiplication rules, this reduces to $\mF(\vtheta_k) = \nabla_{\vtheta}^{\top}\hat{\vy}_k \cdot \nabla_{\hat{\vy}}^2 \mathcal{L}_k \cdot \nabla_{\vtheta}\hat{\vy}_k$.
Identifying $\mH_k = \nabla_{\vtheta}\hat{\vy}_k$ and noting that for Gaussian (or more generally exponential-family) likelihoods, $\mR_k^{-1} = \nabla_{\hat{\vy}}^2 \mathcal{L}_k$, we conclude that $\mF(\vtheta_k) = \mH_k^{\top} \mR_k^{-1} \mH_k$. 
Substituting into the regularized Fisher information, we obtain $\tilde{\mF}_k = \mH_k^{\top} \mR_k^{-1} \mH_k + \rho \mH_k^{\top} \mH_k$, or equivalently $\tilde{\mF}_k = \mH_k^{\top} (\mR_k^{-1} + \rho \mI) \mH_k$.
This motivates the definition of a modified observation noise covariance, $ \tilde{\mR}_k = (\mR_k^{-1} + \rho \mI)^{-1}$, which can be interpreted as a regularized version of $\mR_k$. To avoid repeated inversion, we factor out $\mR_k^{-1}$, yielding $\tilde{\mR}_k = \mR_k (\mI + \rho \mR_k)^{-1}$. 
Substituting $\tilde{\mR}_k$ into the Kalman gain expression (see \autoref{Eq. Algorithm diag gain update}), the regularized Kalman (R-Kalman) gain becomes:
% 
% \vspace{-0.1cm}
\begin{equation}
\tilde{\mK}_k = \mSigma_{k|k-1} \mH_k^{\top} \left(\mH_k \mSigma_{k|k-1} \mH_k^{\top} + \mR_k (\mI + \rho \mR_k)^{-1} \right)^{-1} .
\end{equation}

Here, $\tilde{\mK}_k$ represents the regularized Kalman gain.
For small $\rho \ll 1$, computational efficiency can be improved by approximating the inverse via a first-order Neumann series: $(\mI + \rho \mR_k)^{-1} = \mI - \rho \mR_k + \mathO(\rho^2)$.
The pseudocode of the proposed R-Kalman algorithm is provided in \autoref{Algorithm}, and its computational complexity is analyzed in Appendix~\autoref{Bayesian Computational Complexity}.

\begin{figure*}[t]
% \begin{flushleft}
    \includegraphics[width=\textwidth]{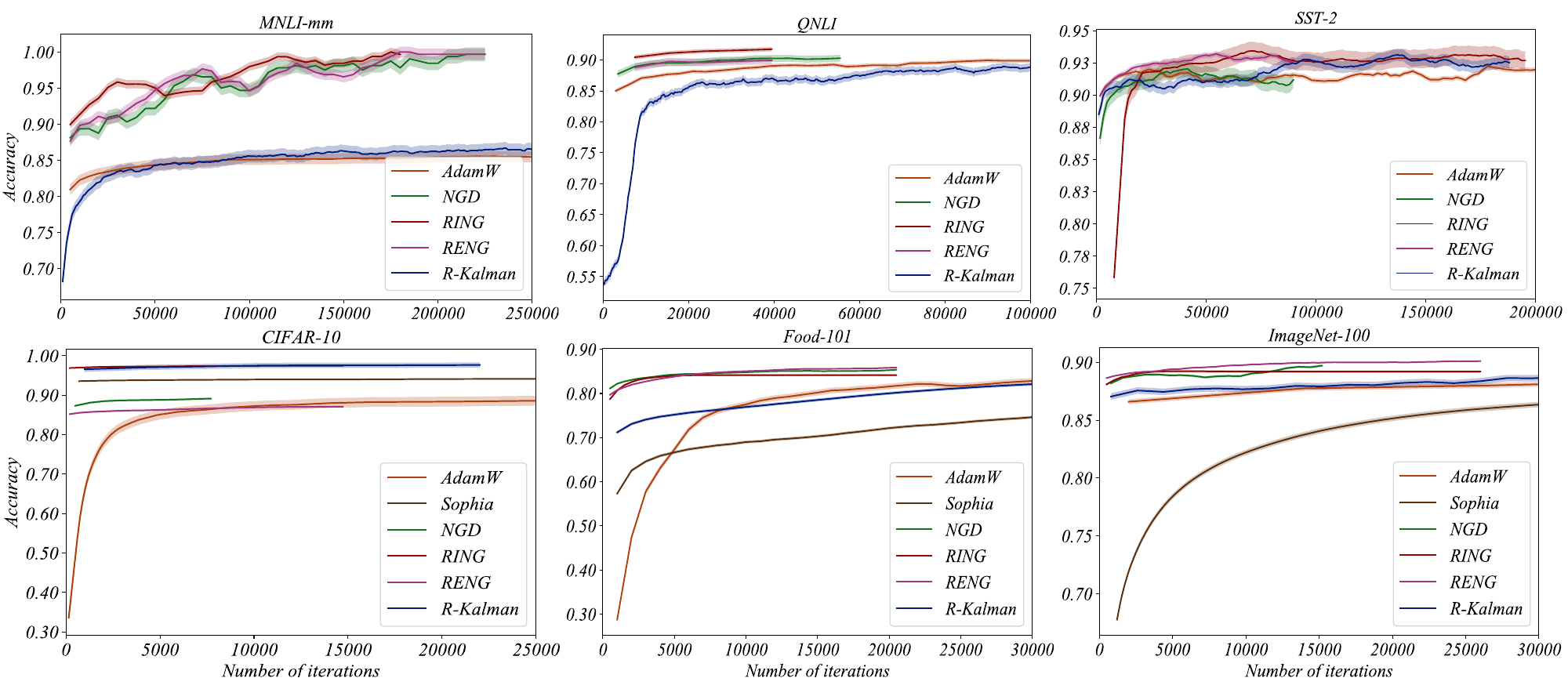}
    \caption{Validation accuracy of our proposed algorithms (RING, RENG, and R-Kalman) compared to AdamW, Sophia, and NGD on selected language (top row) and vision (bottom row) datasets. The plots illustrate validation accuracy as a function of training iterations. For a comprehensive set of results, please refer to \autoref{sec:Additional Experiments}.}
    \label{fig:main_results}
% \end{flushleft}
\end{figure*}

\section{Convergence of Gradient Regularized Natural Gradients}
\label{sec:GR_convergence}

In this section, we provide a convergence analysis of GRNG for full-batch training of a two-layer neural network in the over-parameterized regime.
We start with the following standard assumptions \cite{zhang2019fast}:

\begin{assumption}
\label{assumption1}
    \textbf{Full row rank of the Jacobian matrix}:
    The Jacobian matrix $\mH_0$ at the initialization has full row rank, or equivalently, the Gram matrix $\mG_0 = \mH_0 \mH_0^\top$ is positive definite.
\end{assumption}

\begin{assumption}
\label{assumption2}
    \textbf{Stable Jacobian}:
    There exists $0 \leq C < \frac{1}{2}$ such that for all parameters $\vtheta$ that satisfy $\|\vtheta - \vtheta_0 \|_2 \leq \frac{3\|\vy - \hat{\vy}_0\|_2}{\sqrt{\lambda_{\mathrm{min}}(\mG_0)}}$, we have
    \begin{equation}
        \|\mH(\vtheta) - \mH_0\|_2 \leq \frac{C}{3} \sqrt{\lambda_{\mathrm{min}}(\mG_0)} .
    \end{equation}
\end{assumption}

% \begin{theorem}
% \label{onvergence th.}
%     \textbf{Gradient Regularized Natural Gradients}:
%     Let Assumption \ref{assumption1} and Assumption \ref{assumption2} hold. Suppose we optimize a full-batch training with the ground truth of $\vy$ using GRNG. Then for the training iterations $k = 0, 1, 2, ...$, the error is bounded as:
% \begin{equation}
%     \left\| \hat{\vy}_{k+1} - \vy \right\|_2 \leq  (1 + C) M_k \left\| \hat{\vy}_k - \vy \right\|_2 .
% \end{equation}
% where $M_k = \left( \frac{2 + C}{1 + C} + \rho \kappa (\mG) \| \hat{\vy}_{k} - \vy \|^2_2 \right)$; with $\kappa(\mG)$ as condition number, i.e., $\lambda_{max}(\mG)/\lambda_{min}(\mG)$, and $\lambda(\cdot)$ as the eigenvalue operator.
% \end{theorem}

\begin{theorem}
\label{onvergence th.}
    \textbf{Gradient Regularized Natural Gradients}:
    Let Assumption \ref{assumption1} and Assumption \ref{assumption2} hold. Suppose we optimize a full-batch training with the ground truth of $\vy$ using GRNG with a learning rate $\eta$. Then for the training iterations $k = 0, 1, 2, \dots$, the squared error is bounded as:
$$
    \| \vy - \hat{\vy}_{k+1} \|^2_2 \leq  (1 - \eta) \| \vy - \hat{\vy}_k \|^2_2 .
$$
Provided that the learning rate satisfies $\eta \leq \frac{2(1 + \alpha_k)(1 + C) - 1}{(1 + \alpha_k)^2(1 + C)^2}$, where $\alpha_k = 1 + \rho \kappa (\mG) \| \vy - \hat{\vy}_{k} \|^2_2$; with $\kappa(\mG)$ as condition number, i.e., $\lambda_{max}(\mG)/\lambda_{min}(\mG)$, and $\lambda(\cdot)$ as the eigenvalue operator.
\end{theorem}

The detailed proof is provided in \autoref{convergence}. From this per-step reduction in output space, we can get the following corollary, which shows that the step size only depends on a constant $C$ and the maximum output error at any step $k$.

% \begin{corollary}
% \label{corr onvergence th.}
%     From the above theorem, we can state that the $k^{th}$ step error in the output space is bounded from initialization as:
% \begin{equation}
%     \left\| \hat{\vy}_{k+1} - \vy \right\|_2 \leq (1 + C)^k \hat{M}^k \left\| \hat{\vy}_0 - \vy \right\|_2 ,
% \end{equation}
% where $\hat{M} = \max_{i} \left( M_i \right)$.
% \end{corollary}

\begin{corollary}
\label{corr onvergence th.}
    From the above theorem, we can state that if the learning rate satisfies $\eta \leq \hat{M}$, where $\hat{M} = \min_{i} \left( \frac{2(1 + \alpha_i)(1 + C) - 1}{(1 + \alpha_i)^2(1 + C)^2} \right)$, then the $k^{th}$ step error in the output space exhibits geometric decay from initialization as:
$$
    \| \vy - \hat{\vy}_{k} \|^2_2 \leq (1 - \eta)^k \| \vy - \hat{\vy}_0 \|^2_2 .
$$
\end{corollary}

% \begin{figure}
%     \includegraphics[width=0.4\textwidth]{Figure/Other/grad-epoch_final.pdf}
%     \caption{??.}
%     \label{fig:main_results}
% \end{figure}

% \begin{figure}
%     \includegraphics[width=0.4\textwidth]{Figure/Other/pearson-epoch_final.pdf}
%     \caption{??.}
%     \label{fig:main_results}
% \end{figure}

\begin{table*}
  \caption{Test accuracy (mean $\pm$ standard deviation) for fine-tuning ViT-B16 on CIFAR-10/100, Oxford-IIIT Pet, Food-101, and ImageNet-100 using various optimization algorithms. Higher values indicate better performance. }
  \label{Table: Result Vision}
  \centering
  % \renewcommand{\arraystretch}{1.5} % Increase spacing between rows
  % {\small
  \begin{tabular}{lccccc}
    \toprule
    \textbf{Algorithm} & \textbf{CIFAR-10} & \textbf{CIFAR-100}  &  \textbf{Oxford-IIIT Pet} & \textbf{Food-101} & \textbf{ImageNet-100} \\
    \midrule
    \textbf{AdamW} & $88.9_{\pm0.10}$  & $87.8_{\pm0.02}$ & $91.5_{\pm0.08}$ & $85.2_{\pm0.02}$ & $89.0_{\pm0.04}$  \\ \addlinespace[1ex]
    \textbf{Sophia}      & $97.1_{\pm0.05}$ & $86.8_{\pm0.04}$ & $88.8_{\pm0.01}$ & $80.0_{\pm0.01}$ & $88.2_{\pm0.03}$ \\ \addlinespace[1ex]
    \textbf{NGD}      & $89.3_{\pm 0.05}$ & $\boldsymbol{89.3_{\pm0.02}}$ & $\boldsymbol{93.1_{\pm0.02}}$ & $85.4_{\pm0.01}$ & $89.6_{\pm0.02}$  \\ \addlinespace[1ex]
    \textbf{RING} & $91.5_{\pm0.02}$  & $88.2_{\pm0.02}$ & $92.6_{\pm0.01}$ & $85.0_{\pm0.01}$ & $89.0_{\pm0.01}$  \\ 
    \addlinespace[1ex]
    \textbf{RENG} & $89.2_{\pm0.10}$  & $88.2_{\pm0.05}$ & $91.2_{\pm0.01}$ & $85.8_{\pm0.01}$ & $\boldsymbol{90.10_{\pm0.01}}$  \\ \addlinespace[1ex]
    \textbf{R-Kalman} & $\boldsymbol{97.1_{\pm0.04}}$  & $88.06_{\pm0.02}$ & $92.8_{\pm0.02}$ & $\boldsymbol{86.2_{\pm0.01}}$ & $89.8_{\pm0.04}$  \\ 
    \bottomrule
  \end{tabular}
  % }
\end{table*}

\begin{table*}
  \small
  \caption{Test accuracy or Pearson correlation (mean $\pm$ standard deviation) for fine-tuning RoBERTa-base on the GLUE benchmark, including MNLI-mm , QQP, QNLI, SST-2, CoLA, STS-B, MRPC, and RTE using various optimization algorithms. Higher values indicate better performance. }
  \label{Table: Result Language}
  \centering
  {\small
  \begin{tabular}{lcccccccc}
    \toprule
    \textbf{Algorithm} & \textbf{MNLI-mm} & \textbf{QQP}  &  \textbf{QNLI} & \textbf{SST-2} & \textbf{CoLA} & \textbf{STS-B} & \textbf{MRPC} &  \textbf{RTE} \\
    \midrule
    \textbf{AdamW} & $85.5_{\pm0.2}$  & $87.2_{\pm0.3}$ & $90.0_{\pm0.2}$ & $90.4_{\pm0.3}$ & $62.0_{\pm0.3}$ & $88.8_{\pm0.3}$ & $87.8_{\pm0.1}$ & $71.8_{\pm0.2}$  \\ \addlinespace[1ex]
    \textbf{NGD}      & $97.2_{\pm0.5}$ & $88.6_{\pm0.3}$ & $90.6_{\pm0.3}$ & $91.2_{\pm0.5}$ & $56.2_{\pm0.2}$ & $81.1_{\pm0.2}$ & $86.5_{\pm0.3}$ & $71.0_{\pm0.1}$  \\ \addlinespace[1ex]
    \textbf{RING} & $\boldsymbol{98.6_{\pm0.1}}$  & $88.8_{\pm0.1}$ & $\boldsymbol{91.5_{\pm0.1}}$ & $92.4_{\pm0.1}$ & $58.5_{\pm0.0}$ & $88.2_{\pm0.3}$ & $\boldsymbol{88.2_{\pm0.2}}$ & $\boldsymbol{76.5_{\pm0.2}}$ \\ \addlinespace[1ex]
    \textbf{RENG} & $97.6_{\pm0.2}$  & $\boldsymbol{90.2_{\pm0.1}}$ & $90.5_{\pm0.2}$ & $\boldsymbol{92.6_{\pm0.1}}$ & $\boldsymbol{62.2_{\pm0.0}}$ & $89.2_{\pm0.3}$ & $88.1_{\pm0.1}$ & $75.2_{\pm0.1}$ \\ \addlinespace[1ex]
    \textbf{R-Kalman} & $87.5_{\pm0.1}$  & $86.1_{\pm0.2}$ & $90.8_{\pm0.2}$ & $90.9_{\pm0.2}$ & $59.5_{\pm0.0}$ & $\boldsymbol{89.8_{\pm0.2}}$ & $\boldsymbol{88.2_{\pm0.2}}$ & $74.6_{\pm0.2}$  \\ 
    \bottomrule
  \end{tabular}
  }
\end{table*}

% \vspace{-0.2cm}
\section{Experimental Evaluation}
\label{sec: Experiments and Discussion}
% \vspace{-0.2cm}
We evaluate the performance of our proposed algorithms through extensive experiments on well-established computer vision and language datasets, using different models in transfer learning.
% Results are compared against baseline optimizers such as AdamW, Natural Gradient Descent (NGD), and Sophia\mingfei{add ref for each}. The experimental setups incorporate a diverse range of benchmarks and baseline methods for transfer learning, which enables a comprehensive and robust assessment of the proposed algorithms.

\paragraph{Baseline Optimizers.}
We consider AdamW \citep{loshchilov2017decoupled}, Sophia \citep{liu2023sophia}, and NGD \citep{amari1998natural} as baseline algorithms. AdamW is a widely adopted first-order optimization method, serving as a strong practical baseline. NGD represents a classical second-order optimization approach, which we implement using the K-FAC method with Tikhonov regularization. Sophia, in contrast, is a recently proposed second-order optimizer that leverages a diagonal approximation of the Fisher information matrix.
For parameter-efficient fine-tuning, we adopt LoRA \citep{hu2022lora}, following the framework in \citep{houlsby2019parameter, han2024parameter}. LoRA has emerged as a widely used technique due to its ability to enable efficient model adaptation while maintaining low computational and memory overhead \citep{han2024parameter}.

% \vspace{-0.4cm}

\paragraph{Evaluation Metrics.}
To assess the effectiveness of the proposed algorithms, we adopt standard evaluation metrics corresponding to each benchmark. For image classification tasks, we report Top-1 accuracy as the primary performance measure. For the GLUE benchmark, we follow the established evaluation protocol, using accuracy for most tasks, except for CoLA, where we report Matthews correlation coefficient to capture similarity prediction quality. In addition to predictive performance, we evaluate the computational efficiency of the algorithms by measuring their total running time and comparing it against the corresponding baselines.

\begin{table*}
  \small
  \caption{Time Analysis: Total running-time (and per-iteration time), both in seconds, for fine-tuning of MNLI-mm, QQP, QNLI, CIFAR-100, Food-101, and ImageNet-100 on A100 GPU. Lower values indicate faster convergence.}
  \label{Table: Total time}
  \centering
  {\small
  \begin{tabular}{lcccccc}
    \toprule
    \textbf{Algorithm} & \textbf{MNLI-mm} & \textbf{QQP}  &  \textbf{QNLI} & \textbf{CIFAR-100} & \textbf{Food-101} & \textbf{ImageNet-100} \\
    \midrule
    \textbf{AdamW} & $10308_{(0.21)}$  & $10460_{(0.12)}$ & $\boldsymbol{3011}_{(0.23)}$ & $10156_{(0.65)}$ & $34466_{(0.91)}$ & $51857_{(1.00)}$  \\ \addlinespace[1ex]
    \textbf{NGD}      & $13008_{(2.36)}$ & $12052_{(2.05)}$ & $6939_{(2.12)}$ & $2780_{(2.78)}$ & $5817_{(3.84)}$ & $10600_{(5.11)}$  \\ \addlinespace[1ex]
    \textbf{RING} & $\boldsymbol{8835}_{(1.44)}$  & $8357_{(1.39)}$ & $4811_{(1.47)}$ & $945_{(1.89)}$ & $\boldsymbol{4726}_{(3.12)}$ & $\boldsymbol{8483}_{(4.09)}$ \\ \addlinespace[1ex]
    \textbf{RENG} & $9081_{(1.48)}$  & $8641_{(1.45)}$ & $4975_{(1.52)}$ & $1990_{(1.99)}$ & $4938_{(3.26)}$ & $9686_{(4.67)}$ \\ \addlinespace[1ex]
    \textbf{R-Kalman} & $8994_{(0.26)}$  & $\boldsymbol{8279}_{(0.28)}$ & $3634_{(0.32)}$ & $\boldsymbol{817}_{(1.34)}$ & $4817_{(1.45)}$ & $10150_{(1.28)}$  \\ 
    \bottomrule
  \end{tabular}
  }
\end{table*}

% \vspace{-0.4cm}

\subsection{Experimental Setup}
\label{sec: Experiments Setup}

\paragraph{Datasets.}
We employ the CIFAR-10/100 \citep{krizhevsky2009learning}, Oxford-IIIT Pet \citep{parkhi2012cats}, Food-101 \citep{bossard2014food}, and ImageNet \citep{deng2009imagenet} datasets for image classification tasks due to their wide adoption as standard benchmarks that vary in complexity, scale, and domain diversity. For text classification, we evaluate performance on the GLUE benchmark \citep{wang-etal-2018-glue}, specifically MNLI-mm, QQP, QNLI, SST-2, CoLA, STS-B, MRPC, RTE. These datasets cover a broad spectrum of linguistic phenomena and task types, including natural language inference, paraphrase detection, and sentiment analysis.

% \vspace{-0.2cm}

\paragraph{Pre-trained Models.}
For image-based transfer learning, we utilize the ViT-B16 model \citep{dosovitskiy2020image}, pretrained with DINOv2 on ImageNet \citep{oquab2023dinov2}. For language tasks, we adopt the RoBERTa-base model \citep{liu2019roberta}, which is pretrained on large-scale text corpora.
ViT-B16 is a widely adopted Vision Transformer \citep{alexey2020image} architecture that serves as a standard benchmark in computer vision. Similarly, RoBERTa-base is among the most commonly used pretrained language models in natural language processing.

\paragraph{Training Settings.}
For tasks using AdamW and Sophia, we perform a hyperparameter sweep over learning rates in the range $(10^{-5}, 10^{-3})$, training epochs in $(5, 40)$, and batch sizes in $(8, 64)$. In the case of Adam, a learning rate schedule with a constant warmup phase (100 to 500 iterations) improves convergence and is therefore employed. For NGD, we observe that a constant learning rate close to $1.0$ is often effective. Consequently, we select among the values $\{0.9, 0.99, 0.9999, 1.0, 1.1\}$. We use a Levenberg-Marquardt-type \citep{martens2015optimizing} damping scheme to update the damping coefficient dynamically based on changes in the loss. Tuning the dampening coefficient is challenging due to its sensitivity to task and initialization. However, we have seen across all experiments that $\rho = 10^{-6}$ for Tikhonov regularized NGD and $\rho = 10^{-4}$ for RING and RENG perform well, so we sweep across $(10^{-6}, 10^{-2})$. For the Kalman implementation, we adopt the approach proposed by \citep{chang2022diagonal} to reduce the computational complexity from quadratic to linear in the number of trainable parameters ($n$). 
LoRA hyperparameters are selected based on empirical effectiveness \cite{hu2021lora}. Specifically, for all algorithms, we set the rank $r = 4$ and scaling factor $\alpha = 4$ for the vision datasets and the rank $r = 8$ and scaling factor $\alpha = 8$ for the language datasets. A sensitivity analysis of the key hyperparameters has been conducted for the proposed algorithms. For further details, please refer to \autoref{sec:Ablation Study}.
Full hyperparameter details are also provided in \autoref{Table: Hyperparameters Vision} and \autoref{Table: Hyperparameters Language} in \autoref{Hyperparameters}.
All experiments are conducted on A100 GPU platform, and the average and standard deviation over 5 runs with different initial random seeds are reported.

% \vspace{-0.4cm}

\subsection{Main Results and Discussions}
\label{sec:Results and Discussion}

In this section, we present the experimental results on the MNLI-mm, QNLI, and SST-2 language datasets, as well as the CIFAR-100, Food-101, and ImageNet-100 vision datasets. Results for the remaining datasets are provided in \autoref{sec:Additional Experiments}.
We employ two evaluation metrics in this study. \autoref{fig:main_results} in this section, along with \autoref{fig:nlp_results} and \autoref{fig:image_results} in \autoref{sec:Additional Experiments}, illustrate the validation accuracy curves during training as a function of the number of iterations.

In addition, \autoref{Table: Result Vision} and \autoref{Table: Result Language} report the final test accuracies, and \autoref{Table: Total time} provides time analysis of our algorithms compared to the baselines.
In general, the results demonstrate that our proposed gradient-regularized natural gradient optimizers achieve performance that is comparable or superior to established baselines in fine-tuning tasks.

% \vspace{-0.1cm}

\paragraph{Image Classification Experiments.}
We report the vision experimental results in \autoref{Table: Result Vision}, and \autoref{fig:image_results} of \autoref{sec:Additional Experiments}. As shown, R-Kalman achieves the best test accuracy on CIFAR-10 $(97.1\pm0.04)$, and Food-101 $(86.2\pm0.01)$. RENG demonstrates the best performance on ImageNet-100 $(90.10\pm0.01)$. Whereas, NGD shows better performance on CIFAR-100 $(89.3\pm0.02)$ and Oxford-IIIT Pet $(93.1\pm0.02)$.

% \vspace{-0.1cm}

\paragraph{GLUE Benchmark Experiments.}
We report the GLUE benchmark results in \autoref{Table: Result Language}, and \autoref{fig:nlp_results} of \autoref{sec:Additional Experiments}.
As indicated, RING achieves the highest test accuracy on multiple datasets, including MNLI-mm $(98.6\pm0.1)$, QNLI $(91.5\pm0.1)$, MRPC $(88.2\pm0.2)$, and RTE $(76.5\pm0.2)$.
RENG also outperforms on several tasks, and shows the best performance on QQP $(90.2\pm0.1)$, SST-2 $(92.6\pm0.1)$, and CoLA $(62.2\pm0.0)$. And, R-Kalman demonstrates higher performance in SST-B $(89.8\pm0.2)$ and MRPC $(88.2\pm0.2)$, which matches RING. 

% \vspace{-0.2cm}

\paragraph{Overall.}
We observe that full Fisher-based algorithms, such as RING, RENG, and NGD, tend to perform better on high-volume datasets, including ImageNet, CIFAR-100, MNLI-mm, QQP, QNLI, and SST-2. This can be attributed to the fact that accurately approximating the full Fisher information matrix typically requires $m \approx \mathcal{O}(n log(n))$ samples in a batch, where $n$ denotes the number of trainable parameters \citep{vershynin2011introductionnonasymptoticanalysisrandom}.
Moreover, R-Kalman, as a Bayesian approach, outperforms its frequentist, gradient-based counterparts when applied to datasets with limited training data. This advantage arises because gradient-based optimizers—especially those that lack effective regularization—are prone to overfitting and often produce overconfident predictions. In contrast, Bayesian algorithms inherently model uncertainty, which leads to improved generalization and a reduced risk of overfitting, as also discussed in \citep{wilson2020bayesian}.

% \vspace{-0.2cm}

\paragraph{Time Analysis.}
The time analysis presented in \autoref{Table: Total time} demonstrates that the proposed algorithms—RING, RENG, and R-Kalman—yield significantly faster total running times compared to the AdamW and NGD baselines across most datasets. RING achieves the fastest convergence on MNLI-mm, Food-101, and ImageNet-100, while R-Kalman is superior on QQP and CIFAR-100. Notably, despite AdamW generally exhibiting the lowest per-iteration overhead, the superior performance of our algorithms is attributed to their faster convergence rates (lower total running-time), and this confirms their effectiveness in substantially accelerating the fine-tuning process across diverse language and computer vision datasets. For a detailed analysis of the computational complexity, see \autoref{Computational Complexity}.

\section{Conclusion}
\label{Conclusion}

% \vspace{-0.2cm}

In this study, we proposed gradient-regularized natural gradient optimizers that extend the classical natural gradient descent framework by incorporating gradient regularization, and can be formulated within both Frequentist and Bayesian paradigms. Our algorithms achieve efficient parameter updates without the computational overhead of directly inverting the Fisher information matrix. By leveraging Newton’s iteration in the Frequentist setting and a Kalman-based approach in the Bayesian setting, our methods closely approximate the gradient-regularized natural gradient directions while maintaining computational feasibility. Extensive benchmarking on diverse image and language datasets demonstrates that our optimizers, particularly R-Kalman in low-data regimes and RING/RENG in large-data regimes, achieve comparable or superior performance to established baselines like Adam, Sophia, and NGD in fine-tuning tasks. 
% It is worth noting that our analysis primarily focuses on fine-tuning scenarios; the generalizability of our methods to other domains, such as pre-training, remains an open question. Addressing this limitation and extending the applicability of our methods to broader learning settings will be an important direction for future research.

% \bibliographystyle{plainnat}
% \bibliographystyle{aistats2026}
\bibliography{References}

%%%%%%%%%%%%%%%%%%%%%%%%%%%%%%%%%%%%%%%%%%%%%%%%%%%%%%%%%%%%
\clearpage
\section*{Checklist}

% % %%% BEGIN INSTRUCTIONS %%%
% The checklist follows the references. For each question, choose your answer from the three possible options: Yes, No, Not Applicable.  You are encouraged to include a justification to your answer, either by referencing the appropriate section of your paper or providing a brief inline description (1-2 sentences). 
% Please do not modify the questions.  Note that the Checklist section does not count towards the page limit. Not including the checklist in the first submission won't result in desk rejection, although in such case we will ask you to upload it during the author response period and include it in camera ready (if accepted).

% \textbf{In your paper, please delete this instructions block and only keep the Checklist section heading above along with the questions/answers below.}
% % %%% END INSTRUCTIONS %%%

\begin{enumerate}

  \item For all models and algorithms presented, check if you include:
  \begin{enumerate}
    \item A clear description of the mathematical setting, assumptions, algorithm, and/or model. [Yes]
    \item An analysis of the properties and complexity (time, space, sample size) of any algorithm. [Yes. Please refer to \autoref{Computational Complexity}]
    \item (Optional) Anonymized source code, with specification of all dependencies, including external libraries. [No. Some parts of the implementation depend on confidential components from a collaborative project and cannot be shared publicly. We will release the code once it can be made publicly available.]
  \end{enumerate}

  \item For any theoretical claim, check if you include:
  \begin{enumerate}
    \item Statements of the full set of assumptions of all theoretical results. [Yes]
    \item Complete proofs of all theoretical results. [Yes. Please refer to \autoref{Proofs}]
    \item Clear explanations of any assumptions. [Yes]     
  \end{enumerate}

  \item For all figures and tables that present empirical results, check if you include:
  \begin{enumerate}
    \item The code, data, and instructions needed to reproduce the main experimental results (either in the supplemental material or as a URL). [Yes. Please refer to \autoref{sec: Experiments Setup}, \autoref{Algorithm}, and \autoref{Hyperparameters}]
    \item All the training details (e.g., data splits, hyperparameters, how they were chosen). [Yes. Please refer to \autoref{sec: Experiments Setup}, \autoref{Algorithm}, and \autoref{Hyperparameters}]
    \item A clear definition of the specific measure or statistics and error bars (e.g., with respect to the random seed after running experiments multiple times). [Yes]
    \item A description of the computing infrastructure used. (e.g., type of GPUs, internal cluster, or cloud provider). [Yes]
  \end{enumerate}

  \item If you are using existing assets (e.g., code, data, models) or curating/releasing new assets, check if you include:
  \begin{enumerate}
    \item Citations of the creator If your work uses existing assets. [Yes]
    \item The license information of the assets, if applicable. [Not Applicable]
    \item New assets either in the supplemental material or as a URL, if applicable. [Not Applicable]
    \item Information about consent from data providers/curators. [Yes]
    \item Discussion of sensible content if applicable, e.g., personally identifiable information or offensive content. [Not Applicable]
  \end{enumerate}

  \item If you used crowdsourcing or conducted research with human subjects, check if you include:
  \begin{enumerate}
    \item The full text of instructions given to participants and screenshots. [Not Applicable]
    \item Descriptions of potential participant risks, with links to Institutional Review Board (IRB) approvals if applicable. [Not Applicable]
    \item The estimated hourly wage paid to participants and the total amount spent on participant compensation. [Not Applicable]
  \end{enumerate}

\end{enumerate}

\clearpage
\appendix
\thispagestyle{empty}

% Supplementary material: To improve readability, you must use a single-column format for the supplementary material.
\onecolumn
\aistatstitle{Technical Appendices and Supplementary Material}

\section{Related work}
\label{Related work}

\subsection{Gradient Regularization}
\label{app:Gradient Regularization}

The generalization of deep neural networks has been closely linked to implicit and explicit regularization mechanisms, which emerge from factors such as finite-step stochastic updates and the use of relatively small batch sizes during optimization. A central theme in prior work has been devoted to connecting properties of the loss landscape—particularly the sharpness of local minima, as well as the flatness and curvature of the loss surface—to the generalization behavior of SGD.

A substantial body of work has examined how SGD induces implicit gradient regularization, which favors flat minima and thereby improves generalization \cite{neyshabur2017implicit} \cite{chaudhari2018stochastic}. Some studies have investigated explicit and implicit gradient regularization of SGD in deep learning \cite{smith2021on} \cite{barrett2020igr}. Using Euler discretization of stochastic differential equations for SGD and backward error analysis, they found that the discrete-time update of the usual gradient descent implicitly regularizes the norm of stochastic gradients when its dynamics are mapped to the continuous-time counterpart. To identify these so-called "flat-minima," previous works have also endeavoured towards finding different generalization metrics that can classify a sharp/flat minima. Path-SGD proposed by \cite{neyshabur2015path}, which uses a path-wise norm regularization to achieve better generalization, albeit incurring computational overhead in estimating such a factor. Frobenius norm and spectral norm of the second-order preconditioner, specifically the Hessian matrix, have also been proposed to classify flat minima \cite{wu2017towards} \cite{keskar2017improving}. The local entropy of the loss landscape has also been proposed by \cite{chaudhari2019entropy}, which defines an entropy-guided SGD and shows faster convergence, but for small models on small datasets. Recently, \cite{jastrzebski2021catastrophic} and \cite{karakida2023understanding} have shown that implicit regularization effects of large learning rates can be well explained by the trace of the Fisher information matrix from an early phase of training and have empirically demonstrated an effect called \textit{catastrophic Fisher explosion} for ResNet models on a subset of ImageNet. Work by \cite{jia2020information} proposes a regularizer based on the determinant of FIM, which can bias the optimization trajectory towards finding good local minima (minima with better generalization) and also provide a PAC-based generalization bound.

\subsection{Frequentist Approach:}
\label{app:Frequentist Optimization}

In the Frequentist optimization approach, a loss function is minimized, typically by gradient-based methods. Gradients are treated as deterministic quantities, and optimization algorithms compute the steepest descent direction using gradient-based first-order or second-order algorithms.

Stochastic gradient descent (SGD) \citep{robbins1951sgd} and its variants have long been the backbone of training large-scale neural networks. Among first-order methods, momentum-based techniques like Nesterov Accelerated Gradient (NAG) \citep{nesterov1983method}, and adaptive algorithms such as Adam \citep{kingma2014adam}, Adagrad \citep{duchi2011adaptive}, AdaDelta \citep{zeiler2012adadelta}, RMSProp \citep{graves2013rmsprop}, RADAM \citep{liu2019variance-radam}, and AdaBelief \citep{adabelief_zhuang} have gained considerable popularity due to their ability to adapt learning rates to the geometry of the loss landscape.
While SGD is simple and memory-efficient, its constant learning rate can lead to inconsistent parameter updates in the inherently non-Euclidean geometry of deep neural networks. Adaptive optimizers address this by adjusting the learning rate based on local curvature, often leading to faster convergence and improved performance in certain settings. For example, Adam uses the square root of the diagonal of the estimated second moment (variance) of gradients to scale the update direction. However, this often results in instability at high learning rates.
To address such issues, RADAM introduces a variance rectification mechanism to justify the warm-up phase in Adam. 

Empirical studies highlight that while SGD (with or without momentum) performs competitively in computer vision tasks, adaptive optimizers like Adam significantly outperform SGD in training transformer-based models for NLP tasks \citep{zhang2024transformers}. However, the effectiveness of Adam remains poorly understood. One explanation is the presence of heavy noise in attention layers \citep{adaptive_attention}, which affects the performance of SGD more severely. \citep{kunstner2023noise} observed that Adam continues to outperform SGD even with large batch sizes, attributing this to the signed nature of Adam’s updates. Indeed, Adam can be seen as a variance-adapted signed descent method \citep{adabelief_zhuang, pmlr-v80-balles18a}.
Another hypothesis is rooted in the geometry of the optimization problem. \citep{zhang2024transformers} argues that the heterogeneous Hessian spectrum in attention layers contributes to slow convergence under SGD. This aligns with the theoretical insights of \citep{thomas2020interplay}, who noted that the Hessian reflects curvature under the data distribution, which may result in poor generalization and data inefficiency.

Given the limitations of both SGD and Adam, second-order information has been explored to better precondition gradient updates. While the Hessian provides curvature information, its computation is expensive and, more critically, it may not be positive semi-definite due to the non-convex nature of deep learning objectives. Indeed, empirical studies \citep{sagun2016eigenvalues, sagun2017empirical, pmlr-v97-ghorbani19b-hesseigdensity} show that the Hessian spectrum contains negative eigenvalues, particularly near convergence.
An alternative is the Fisher Information Matrix (FIM), which is always positive semi-definite and defined as the covariance of gradients of the negative log-likelihood loss under the model distribution \citep{amari1998natural}. The natural gradient \citep{exactngd} leverages the inverse FIM to scale gradients appropriately. Although computing and inverting the FIM is computationally expensive, several approximation methods have made natural gradient descent feasible. K-FAC \citep{martens2015optimizing} decomposes the FIM into Kronecker factors for efficient inversion. Shampoo \citep{gupta2018shampoo} estimates curvature using stochastic gradients, while \citep{zhang2023eva} proposes a memory-efficient rank-one update scheme.
More recently, \citep{liu2023sophia} proposed Sophia, which leverages a diagonal approximation of the Hessian and employs coordinate-wise clipping to mitigate rapid changes in curvature due to small batch sizes. Similarly, \citep{pan2023understanding} emphasized the role of directional sharpness in adaptive optimizers and showed that coordinate-wise clipping consistently improves optimization stability.
Overall, these approaches highlight the importance of leveraging second-order information—particularly the Fisher matrix—in training large models like transformers, where curvature is both heterogeneous and evolving.

Beyond optimization speed, several studies have also investigated how optimization strategies influence generalization. Regularization techniques such as Double Backpropagation \citep{doublebp} were early attempts to control curvature by penalizing gradient norms. More recent works \citep{Sankar_2021_HessianEigenspec} suggest that regularizing the trace of the Hessian improves generalization in SGD. Similarly, \citep{barrett2020igr} demonstrates that gradient regularization can promote flatter minima, thereby enhancing robustness to noise and improving test-time performance.

\subsection{Bayesian Approach:}
\label{app:Bayesian Optimization}

Bayesian approaches interpret the trainable parameters as random variables and aims to infer their posterior distribution given the data by applying Bayes’ theorem \citep{murphy2023probabilistic, murphy2012machine}, which is known as \textit{exact Bayesian inference}.
The exact Bayesian inference is computationally intractable for high-dimensional deep models due to the need to integrate over the entire possible values in the estimation of the posterior and predictive distribution. Consequently, approximate methods are employed, which can be broadly classified into sampling-based and approximate-inference approaches \citep{khan2023bayesian}.

Sampling methods like Markov Chain Monte Carlo (MCMC) \citep{gelfand1990sampling}, including Hamiltonian Monte Carlo (HMC) \citep{neal2012bayesian} and Metropolis-Hastings (MH) \citep{chib1995understanding} work by building a Markov chain whose equilibrium distribution matches the target posterior. Then, one can generate samples that approximate the posterior distribution by running the chain for many iterations. These sampling-based approaches are either infeasible or slow for deep learning with a large parameter space \citep{khan2023bayesian}.

Approximate-inference methods like Variational Inference (VI) \citep{jordan1999introduction} approximate the posterior with a simpler distribution $q(\vpsi)$ (e.g., a Gaussian) and optimize its parameters $\vpsi$ to minimize the KL divergence of $q(\vpsi)$ given $p(\vtheta \mid \mathD)$. This is equivalent to minimizing the variational loss (or negative ELBO) $\mathcal{L}(\psi) = -\sE_{q(\vpsi)}[\text{log } p(\mathD \mid \vtheta)]-\text{KL}(q(\vpsi) \mid p(\vtheta))$ \citep{li2024uncertainty}. 
Laplace algorithm \citep{mackay1992bayesian}, as another approximate-inference method, approximates the posterior as a Gaussian with the mean at the maximum a posteriori (MAP) estimate, and covariance of the Hessian of the log-posterior.
These approximate inference algorithms use a Frequentist optimization approach, mostly first-order gradient-based techniques, which rely solely on gradient information, like \textit{Bayes by Backprop} introduced by \citep{blundell2015weight}.
While many early methods relied on first-order optimization, recent studies have demonstrated the advantages of second-order methods for optimizing the variational loss. Works such as \citep{lin2024can}, \citep{shen2024variational}, and \citep{khan2023bayesian} employ second-order optimization techniques and leverage the curvature information of the variational loss to enable more informative updates, leading to faster convergence and better uncertainty estimation.

The Extended Kalman Filter (EKF) provides a straightforward solution for approximate Bayesian inference, and enables closed-form updates for the posterior distribution over neural network parameters \citep{jones2024bayesian}. Recent studies \citep{ollivier2018online, ollivier2019extended} demonstrate that the EKF update rule is equivalent to an online natural gradient descent.
This effectively offers a new perspective into second-order optimization: instead of computing various large pre-conditioning matrices and their inversions, we can recursively infer the natural direction update as the parameter update step in the Kalman algorithm.
In contrast to NGD, the Kalman formulation doesn't require any FIM inversion, which is computationally expensive for large models. Although promising, the practical implementation of such techniques remained infeasible for large models with millions of parameters because of the huge size of the matrix involved in the Kalman algorithm. Various studies have since addressed the scalability challenges of the EKF optimizer. 
For instance, \citep{chang2022diagonal, abdi2025bayesian, abdi2024loko} introduced a diagonal covariance matrix approximation, while \citep{chang2023low} proposed a low-rank plus diagonal decomposition of the posterior precision matrix. \citep{hennig2024computation} developed a matrix-free iterative algorithm to further enhance efficiency. 

In parallel, other works have explored the intersection of Bayesian inference and natural gradient methods. The Topmoumoute online natural gradient algorithm \citep{roux2007topmoumoute} pioneered an online formulation of the natural gradient with connections to information geometry. Building on this idea, \citep{khan2023bayesian} bridges natural gradient descent and Bayesian learning by proposing a general update rule based on the principle of minimizing the divergence between successive posterior distributions. More recently, \citep{jones2024bayesian} extended this framework to variational inference in neural networks and introduced the Bayesian Online Natural Gradient (BONG) algorithm.

\clearpage

\section{Algorithm}
\label{Algorithm}

The pseudocode of the proposed algorithms has been shown in Algorithm \ref{alg:ring-reng} and Algorithm \ref{alg:kalman}.

\begin{algorithm}[H]
\caption{\textcolor{blue}{RING:} Regularized Implicit Natural Gradient\\
\hspace{2em} \textcolor{green}{RENG:} Regularized Explicit Natural Gradient}
\label{alg:ring-reng}
\begin{algorithmic}[1]
    \STATE {\bf Initialization:} $\mW_i$ $\forall{i \in [1,L]}$, number of weight layers $L$, learning rate $\alpha$, gradient regularizer $\rho$, lambda $\lambda$, lambda discount factor $\phi$, Fisher skip frequency $S$.
    \FOR{$k = 1, 2, ...$}
    \STATE Compute minibatch loss $\mathcal{L}_k$ and gradients $\nabla \ln p(\vy|\vx,\vtheta)$.
    \IF{ $k\ \mathrm{mod}\ S = 0$}
    \STATE \textcolor{green}{RENG:} Use the gradient to calculate the regularized loss in \autoref{eq:grad_reg_loss}
    \STATE Save the activation matrix $\sE_{p(\vy|\vx,\vtheta)}[ \boldsymbol{x}_{i-1} \boldsymbol{x}_{i-1}^{\top} ]$ and batch gradients $\sE_{p^*(\vx,\vy)}[ \boldsymbol{e}_i \boldsymbol{x}_{i-1}^{\top} ]$.
    \STATE  Sample output $\vy \sim p(\vy|\vx,\vtheta)$ and compute gradient.
    \STATE Save the error matrix $\sE_{p(\vy|\vx,\vtheta)}[ \boldsymbol{e}_i \boldsymbol{e}_i^{\top} ]$.
    \STATE Update $\mW_i$ using \textcolor{blue}{RING:} \autoref{Eq.RING}, or \textcolor{green}{RENG:} \autoref{Eq. RENG}.
    \ELSE 
    \STATE Update regularized activation and error matrices via \autoref{eq:inverse_update}.
    \STATE Update $\mW_i$ using \textcolor{blue}{RING:} \autoref{Eq.RING}, or \textcolor{green}{RENG:} \autoref{Eq. RENG}.
    \ENDIF
    \ENDFOR
\end{algorithmic}
\end{algorithm}

\begin{algorithm}
\caption{Regularized Kalman Algorithm}
\label{alg:kalman}
\begin{algorithmic}[1]
\STATE \textbf{Initialization:}
    $p(\vtheta_0) = \mathN(\vmu_0,\mSigma_0)$,   $\mR_{0} = \mO_{d_o \times d_o}, \mQ = \mO_{d_i \times d_i}$\\
    \FOR{$k = 1, 2, ... $}
        \STATE \textbf{Prediction:}
        \STATE $\vmu_{k|k-1} = \vmu_{k-1}$
        \STATE $\mSigma_{k|k-1} = \mSigma_{k-1} + \mQ$
        % \STATE
        \STATE \textbf{Pre-Updating:}
        \STATE $\hat{\vy}_k \simeq h(\vx_{k}, \vmu_{k \mid k-1})$
        \STATE $\mH_k = \nabla_{\vtheta}h|_{(\vx_{k}, \vmu_{k \mid k-1})}$
        % \STATE $d_M = \sqrt {\left(\vy_k - \hat{\vy}_k\right) \mR_{k-1}^{-1} \left(\vy_k - \hat{\vy}_k\right)^{\top}}$
        % \STATE $\lambda = \mathbf{e}^{-\alpha d_M}$
        \STATE $\hat{\mR}_k = \left(\vy_k - \hat{\vy}_k\right)\left(\vy_k - \hat{\vy}_k\right)^{\top} + \mH_k \mSigma_{k|k-1} \mH_k^{\top}$
        \STATE $\mR_k = \beta \mR_{k-1} + (1-\beta) \hat{\mR}_k $
        % \STATE
        \STATE \textbf{Updating:}
        \STATE $\tilde{\mK}_k = \mSigma_{k|k-1} \mH_k^{\top} \left(\mH_k \mSigma_{k|k-1} \mH_k^{\top} + \mR_k (\mI + \rho \mR_k)^{-1} \right)^{-1}$
        \STATE $\vmu_{k} = \vmu_{k|k-1} + \tilde{\mK}_k (\vy_k - \hat{\vy}_k)$
        \STATE  $\mSigma_{k} = \mSigma_{k|k-1} - \tilde{\mK}_k \mH_k \mSigma_{k|k-1}$
        % \STATE
        \STATE \textbf{Output:}
         Posterior: $p(\vtheta_k \mid \mathD_{1:k}) = \mathN(\vmu_k, \mSigma_k)$
    \ENDFOR
\end{algorithmic}
\end{algorithm}

\clearpage

\section{Convergence Analysis of Theorem \ref{onvergence th.}}
\label{convergence}
 In this section, we provide detailed proof stating that the output of a two-layer neural network shows exponential decay with GRNG optimizers.
 
\begin{proof}
    
    Considering $\vy$ as the full-batch ground truth, $\mH$ as Jacobian matrix, and $\mG=\mH \mH^{\top}$ as Gram matrix, we introduce a learning rate factor $\eta$ and calculate the difference of predictions between two consecutive iterations:
    $$
    \hat{\vy}_{k+1} - \hat{\vy}_{k} = \hat{\vy}\left(\vtheta_{k} - \eta (\mF_{k}+\rho\|\nabla \mathL\|^2_2)^{-1}\mH_{k}^\top (\hat{\vy}_{k} - \vy)\right) - \hat{\vy}(\vtheta_{k}) .
    $$
    
    For simplicity, we use $\tilde{\rho}_k$ instead of $\rho\|\nabla \mathL_k\|^2_2$ and later on we use the expression $\nabla \mathL_k = \mH_{k}^\top (\hat{\vy}_{k} - \vy)$ to get the final bound:
    $$
     \hat{\vy}_{k+1} - \hat{\vy}_{k} = \hat{\vy}\left(\vtheta_{k} - \eta \mH_{k}^\top (\mG_{k}+\tilde{\rho}_k \mI)^{-1} (\hat{\vy}_{k} - \vy)\right) - \hat{\vy}(\vtheta_{k})
    $$
    
    $$
    = -\eta \int_{s=0}^1 \left\langle \frac{\partial \hat{\vy}(\vtheta(s))}{\partial \vtheta^\top},  \mH_{k}^\top (\mG_{k}+\tilde{\rho}_k \mI)^{-1}(\hat{\vy}_{k} - \vy) \right\rangle ds
    $$
    
    $$
    \leq -\eta \int_{s=0}^1 \left\langle \frac{\partial \hat{\vy}(\vtheta(s))}{\partial \vtheta^\top},  \mH_{k}^\top (\mG_{k}^{-1}-\tilde{\rho}_k\mG_{k}^{-1}\mG_{k}^{-1})(\hat{\vy}_{k} - \vy) \right\rangle ds
    $$
    
    $$
    = -\eta \int_{s=0}^1 \left\langle \frac{\partial \hat{\vy}(\vtheta(s))}{\partial \vtheta^\top},  \mH_{k}^\top \mG_{k}^{-1}(\hat{\vy}_{k} - \vy) \right\rangle ds  +  \eta \int_{s=0}^1 \left\langle \frac{\partial \hat{\vy}(\vtheta(s))}{\partial \vtheta^\top},  \mH_{k}^\top \tilde{\rho}_k\mG_{k}^{-1}\mG_{k}^{-1}(\hat{\vy}_{k} - \vy) \right\rangle ds
    $$
    
    $$
    = \underbrace{-\eta \int_{s=0}^1 \left\langle \frac{\partial \hat{\vy}(\vtheta_{k})}{\partial \vtheta^\top},  \mH_{k}^\top \mG_{k}^{-1}(\hat{\vy}_{k} - \vy) \right\rangle ds}_{\text{\textcircled{1}}}
    \quad + \underbrace{\eta \int_{s=0}^1 \left\langle \frac{\partial \hat{\vy}(\vtheta_{k})}{\partial \vtheta^\top} - \frac{\partial \hat{\vy}(\vtheta(s))}{\partial \vtheta^\top},  \mH_{k}^\top \mG_{k}^{-1}(\hat{\vy}_{k} - \vy) \right\rangle ds}_{\text{\textcircled{2}}} 
    $$
    
    $$
    + \underbrace{\eta \int_{s=0}^1 \left\langle \frac{\partial \hat{\vy}(\vtheta_{k})}{\partial \vtheta^\top},  \mH_{k}^\top \tilde{\rho}_k\mG_{k}^{-1}\mG_{k}^{-1}(\hat{\vy}_{k} - \vy) \right\rangle ds}_{\text{\textcircled{3}}}
    \quad - \underbrace{\eta \int_{s=0}^1 \left\langle \frac{\partial \hat{\vy}(\vtheta_{k})}{\partial \vtheta^\top} - \frac{\partial \hat{\vy}(\vtheta(s))}{\partial \vtheta^\top},  \mH_{k}^\top \tilde{\rho}_k\mG_{k}^{-1}\mG_{k}^{-1}(\hat{\vy}_{k} - \vy) \right\rangle ds}_{\text{\textcircled{4}}}
    $$
    
    To evaluate the L2 norms, we define $\alpha_k = 1 + \rho \kappa(\mG) \| \vy - \hat{\vy}_{k} \|^2_2$, where $\kappa(\mG)$ is the condition number, i.e., $\lambda_{max}(\mG)/\lambda_{min}(\mG)$; and $\lambda(\cdot)$ is the eigenvalue operator. Observing that terms $\text{\textcircled{1}}$ and $\text{\textcircled{3}}$ do not depend on $s$, and using the spectral bound from \cite{zhang2019fast} for $\text{\textcircled{2}}$ and $\text{\textcircled{4}}$, we can define the respective norms as: 
    \begin{align*}
        \| \text{\textcircled{1}} \|_2 &= \eta \| \vy - \hat{\vy}_{k} \|_2 \\
        \| \text{\textcircled{2}} \|_2 &= \eta C \| \vy - \hat{\vy}_{k} \|_2 \\
        \| \text{\textcircled{3}} \|_2 &= \eta \alpha_k \| \vy - \hat{\vy}_{k} \|_2 \\
        \| \text{\textcircled{4}} \|_2 &= \eta C \alpha_k \| \vy - \hat{\vy}_{k} \|_2
    \end{align*}
    
    Hence, the total magnitude of change between iterations is bounded by:
    $$
    \| \hat{\vy}_{k+1} - \hat{\vy}_{k} \|^2_2 \leq \eta^2 (1 + \alpha_k)^2(1 + C)^2 \| \vy - \hat{\vy}_{k} \|^2_2 .
    $$
    
    To examine convergence relative to the initial error, we expand the squared output error at the $(k+1)^{th}$ iteration:
    \begin{align*}
        \| \vy - \hat{\vy}_{k+1} \|^2_2 &= \| (\vy - \hat{\vy}_{k}) - (\hat{\vy}_{k+1} - \hat{\vy}_{k}) \|^2_2 \\
        &= \| \vy - \hat{\vy}_{k} \|^2_2 + \| \hat{\vy}_{k+1} - \hat{\vy}_{k} \|^2_2 - 2 \langle \vy - \hat{\vy}_{k}, \hat{\vy}_{k+1} - \hat{\vy}_{k} \rangle \\
        &\leq \| \vy - \hat{\vy}_{k} \|^2_2 + \eta^2(1 + \alpha_k)^2(1 + C)^2\| \vy - \hat{\vy}_{k} \|^2_2 - 2\eta(1 + \alpha_k)(1 + C)\| \vy - \hat{\vy}_{k} \|^2_2 \\
        &\leq \left[ 1 + \eta^2(1 + \alpha_k)^2(1 + C)^2 - 2\eta(1 + \alpha_k)(1 + C) \right] \| \vy - \hat{\vy}_{k} \|^2_2 .
    \end{align*}
    
    For the residual error to strictly decay, we require the constant factor to be less than or equal to $(1 - \eta)$:
    $$
    \| \vy - \hat{\vy}_{k+1} \|^2_2 \leq (1 - \eta) \| \vy - \hat{\vy}_{k} \|^2_2 .
    $$
    
    This condition is satisfied when:
    $$
    1 + \eta^2(1 + \alpha_k)^2(1 + C)^2 - 2\eta(1 + \alpha_k)(1 + C) \leq 1 - \eta .
    $$
    
    By rearranging the terms, we find the equivalent upper bound for the learning rate:
    $$
    \eta \leq \frac{2(1 + \alpha_k)(1 + C) - 1}{(1 + \alpha_k)^2(1 + C)^2} .
    $$
    
    As $1 + \alpha_k > 1$ and $1 + C > 1$, this makes the right-hand side expression positive, guaranteeing a feasible step size for global convergence. Furthermore, the denominator grows quadratically and the numerator grows linearly, so the maximum permissible $\eta$ is bounded predominantly by the smallest value of $\alpha_k$. 
\end{proof}

From this, let us define $\hat{M} = \max_k \frac{2(1 + \alpha_k)(1 + C) - 1}{(1 + \alpha_k)^2(1 + C)^2}$. When $\eta$ satisfies $\eta \leq \hat{M}$, the loss exhibits geometric decay relative to the functional output of the initial iteration:
$$
    \| \vy - \hat{\vy}_{k+1} \|^2_2 \leq (1 - \eta)^{k+1} \| \vy - \hat{\vy}_{0} \|^2_2 .
$$
And we finally prove Corollary \ref{corr onvergence th.}.
% Note that based on \cite{zhang2019fast}, we already know that $\textcircled{1}$ and $\textcircled{2}$ are bounded. Next, we bound the norms of $\textcircled{3}$ and $\textcircled{4}$:

% Since $\mG_{k} = \frac{\partial \hat{\vy}(\vtheta_{k})}{\partial \vtheta^\top}$ cancel one of $\mG_{k}^{-1}$ in the $\textcircled{3}$, and using \emph{Lemma 2} in \cite{zhang2019fast}, we will have:
% \[
% \|\textcircled{3} \|_2 \leq \tilde{\rho} \|\mG_0^{-1}\|_2 \|\hat{\vy}_{k} - \vy\|_2 \leq \frac{9 \tilde{\rho} \|\hat{\vy}_{k} - \vy\|_2}{4 \lambda_\mathrm{min}(\mG_0)}
% \]
% where, $\lambda_\mathrm{min}(\cdot)$ is the minimum eigenvalue operator.

% Using the \emph{Stable Jacobian} condition in \cite{zhang2019fast}, we have:
% \[
% \|\textcircled{4} \|_2 \leq \tilde{\rho} \frac{2C}{3} \sqrt{\lambda_\mathrm{min}(\mG_0)} \| \mH_{k}^\top \mG_{k}^{-1}\mG_{k}^{-1}(\hat{\vy}_{k} - \vy) \|_2 \leq \tilde{\rho} \frac{2C}{3} \sqrt{\lambda_\mathrm{min}(\mG_0)} \frac{\|\hat{\vy}_{k} - \vy\|_2}{\lambda^{3/2}_\mathrm{min}(\mG_k)} \leq \tilde{\rho} \frac{9C \|\hat{\vy}_{k} - \vy\|_2}{4 \lambda_\mathrm{min}(\mG_0)}
% \]

% Furthermore, we know that the per-step error of the Newton iteration technique quadratically converges to zero asymptotically. Then, we can bound the error between the approximated and exact Fisher inverse. Thus, our regularized algorithms converge to a stationary solution.

\clearpage

\section{Proofs and Detailed Derivations}
\label{Proofs}
This section presents the proofs and detailed derivations of the proposed algorithms, offering a comprehensive understanding of their underlying principles and mechanisms.

\subsection{Proof of \autoref{eq:inverse_update}}
\label{Proof:pro:inverse_update}

\begin{proof}
    Let us assume that at iteration $k$ we have a damping coefficient $\rho_k$, and have computed the activation and error matrices of each layer. Based on \textit{Lazy Fisher}, we assume that ${\mLambda_{i-1}}_{k+1} \approx {\mLambda_{i-1}}_{k}$, and ${\mGamma_{i-1}}_{k+1} \approx {\mGamma_{i-1}}_{k}$. However, the damping coefficient $\rho_{k+1}$ is allowed to be updated. From matrix differential theory, we know that for an invertible matrix $\mA$, we have $\mA \mA^{-1} = \mI$, then the differential of the inverse is given by: $d\mA^{-1} = - \mA^{-1} d\mA \mA^{-1}$. In our context, $\mA_k$ represents either the regularized activation matrix $\mA_k = \tilde{\mLambda}_{k} = (\mLambda_k + \lambda_k \mI)$ or the regularized error matrix $\mA_k = \tilde{\mGamma}_{k} = (\mGamma_k + \lambda_k \mI)$. Under the \textit{Lazy Fisher} assumption, where the error/activation matrix remains unchanged across consecutive timesteps, and only the damping factor evolves, we approximate the change in the matrix as $d\mA_k \approx (\lambda_{k+1} - \lambda_{k})\mI = d\lambda_k \mI$. Then, we add $d\mA_{k}^{-1}$ to the $\mA_k^{-1}$ to approximate the $\mA_{k+1}^{-1}$ without explicitly computing the inversion of $\mA_{k+1}$. Hence, for each layer $i$, the new inverted matrix can be approximated as $\mA_{k+1}^{-1} \approx \mA_{k}^{-1}+d\mA_k^{-1}$.
\end{proof}

\subsection{Proof of \autoref{eq:grad_weight_update}}
\label{Proof:eq:grad_weight_update}

\begin{remark}
\label{remark:ngd}
    The update direction $\vdelta^{\ast} = -\mF(\vtheta)^{-1} \nabla_{\vtheta} \mathcal{L}(\vtheta)$ is the unique minimizer of the quadratic approximation of the loss function $\mathcal{L}$ around the point $\vtheta$, given by: 
    \begin{equation*}
      \mathcal{L}(\vtheta + \vdelta) = \mathcal{L}(\vtheta) + \vdelta^T \nabla_{\vtheta} \mathcal{L}(\vtheta) + \frac{1}{2} \vdelta^{\top} \mF(\vtheta) \vdelta + \mathO(|\vdelta|^3).
    \end{equation*} 
\end{remark}

\begin{remark}
\label{remark:tikh}
    Considering the regularized loss function:
    \begin{equation*}
    \mathcal{L}_T(\vtheta + \vdelta) = \mathcal{L}(\vtheta) + \vdelta^T \nabla_{\vtheta} \mathcal{L}(\vtheta) + \frac{1}{2} \big[ \vdelta^{\top} \mF(\vtheta) \vdelta + \rho \vdelta^{\top} \vdelta \big] + \mathO(|\vdelta|^3).
    \end{equation*}
     This expression can be rewritten as: 
     \begin{equation*}
    \mathcal{L}_T(\vtheta + \vdelta) = \mathcal{L}(\vtheta) + \vdelta^T \nabla_{\vtheta} \mathcal{L}(\vtheta) + \frac{1}{2} \vdelta^{\top} \big(\mF(\vtheta)+\rho \mI \big) \vdelta + \mathO(|\vdelta|^3).
    \end{equation*}
     Thus, based on Remark \ref{remark:ngd}, the new natural direction will be: 
     \begin{equation*}
    \vdelta^{\ast} = - (\mF(\vtheta) + \rho \mI)^{-1} \nabla_{\vtheta} \mathcal{L}(\vtheta).
    \end{equation*}
\end{remark}

\begin{proof}
    Considering the regularized loss function expressed in \autoref{eq:grad_reg_loss}, and using the Remark \ref{remark:ngd} and \ref{remark:tikh}, one can find the optimum update direction as: 
    \begin{equation*}
    \vdelta^{\ast} = - \Big(\mF(\vtheta) + \rho \norm{\sE_{p(\vy|\vx,\vtheta)}[\nabla_{\vtheta} \mathcal{L}(\vtheta)]}^2_2 \mI \Big)^{-1} \nabla_{\vtheta} \mathcal{L}(\vtheta) .
    \end{equation*}
     It can be expressed by activation and error matrices as: 
     \begin{equation*}
    \vdelta^{\ast} = (\mLambda_i \otimes \mGamma_i + \rho \norm{\mLambda_i \otimes \mGamma_i}  \mI)^{-1} \nabla_{\vtheta} \mathcal{L}(\vtheta) ,
    \end{equation*}
      which is:
      \begin{equation*}
    \vdelta^{\ast} = (\mLambda_i \otimes \mGamma_i + \rho \norm{\mLambda_i} \norm{\mGamma_i}  \mI)^{-1} \nabla_{\vtheta} \mathcal{L}(\vtheta) .
    \end{equation*}
     Again, using the approximation $(A \otimes B + \rho I) \approx (A + \sqrt\rho I) \otimes (B + \sqrt\rho I)$, we have:
     \begin{equation*}
    \vdelta^{\ast} = [\mLambda_i + \sqrt{\rho} \norm{\mLambda_i} \mI]^{-1} \otimes [\mGamma_i + \sqrt{\rho} \norm{\mGamma_i} \mI]^{-1} \nabla_{\vtheta} \mathcal{L}(\vtheta) .
    \end{equation*}
     
\end{proof}

\subsection{Proof of \autoref{Eq. Algorithm Updating}}
\label{Proof:information filter}

\begin{proof}

The information form of the Kalman filter (information filtering) can be derived from its standard formulation.
Substituting the Kalman gain from \autoref{Eq. Algorithm diag gain update}, into the standard formulation of covariance update ($\mSigma_{k} = \mSigma_{k|k-1} - \mK_k \mH_k \mSigma_{k|k-1}$) yields:
\begin{equation*}
\mSigma_{k} = \mSigma_{k|k-1} - \mSigma_{k|k-1} \mH_k^{\top} \left(\mH_k \mSigma_{k|k-1} \mH_k^{\top} + \mR_k\right)^{-1} \mH_k \mSigma_{k|k-1} .
\end{equation*}
The Woodbury matrix identity states: $(\mA + \mU\mC\mV)^{-1} = \mA^{-1}-\mA^{-1}\mU(\mC^{-1}+\mV\mA^{-1}\mU)^{-1}\mV\mA^{-1}$. By identifying $\mA = \mSigma^{-1}_{k|k-1}$, $\mU = \mH_k^{\top}$, $\mC = \mR_k^{-1}$, and $\mV = \mH_k$, we apply the identity to obtain the information form of the covariance update: $\mSigma^{-1}_{k} = \mSigma_{k|k-1}^{-1} + \mH_k^{\top} \mR_{k}^{-1} \mH_k$. 

\end{proof}

\subsection{On the Equivalence of Kalman Filtering and NGD}
\label{Proof:Kalman_NGD}

\begin{lemma}
\label{lemma:modified_kamlan_update}

The update step in the Kalman algorithm, as given in \autoref{Eq. Algorithm Updating Process} is equivalent to: 
\begin{align*} 
    \vmu_{k} = \vmu_{k|k-1} - \mSigma_{k} \nabla_{\vtheta} \mathcal{L}_k
    \label{eq:posterior_mean_update} 
\end{align*}

\end{lemma}

\begin{proof}
    
To prove Lemma \ref{lemma:modified_kamlan_update}, consider the negative log-likelihood loss function at step $k$, for a Gaussian or exponential family distribution, defined as: $\mathcal{L}_k = \frac{1}{2}(\hat{\vy}_k-\vy_k)^{\top} \mR_{k}^{-1} (\hat{\vy}_k-\vy_k)$. Applying the chain rule, the gradient of the loss with respect to the parameters is given by: $\nabla_{\vtheta} \mathcal{L}_k = \nabla_{\vtheta}^{\top} \hat{\vy} \nabla_{\hat{\vy}} \mathcal{L}$. We note that $\nabla_{\vtheta} \hat{\vy} = \mH_k$, and $\nabla_{\hat{\vy}} \mathcal{L} = \mR_k^{-1} (\hat{\vy}_k-\vy_k)$. Therefore, the gradient becomes $\nabla_{\vtheta} \mathcal{L}_k = \mH_k^{\top} \mR_k^{-1} (\hat{\vy}_k-\vy_k)$. Rearranging, we obtain the prediction error as $ (\vy_k-\hat{\vy}_k) = - \mR_k {\mH_k^{\top}}^{-1} \nabla_{\vtheta} \mathcal{L}_k$. Substituting this expression into the Kalman filter for the parameter update \autoref{Eq. Algorithm Updating Process}, we get:
\begin{equation*}
\vmu_{k} = \vmu_{k|k-1} - \mK_k \mR_k {\mH_k^{\top}}^{-1} \nabla_{\vtheta} \mathcal{L}_k .
\end{equation*}

Next, We can replace $\mK_k \mR_k$ with $\mSigma_k\mH_k^{\top}$ because of this identity:
\begin{equation*}
\mK_k \mR_k = \mK_k(\mR_k +\mH_k \mSigma_{k|k-1} \mH_k^{\top}) - \mK_k \mH_k \mSigma_{k|k-1} \mH_k^{\top},
\end{equation*}
where based on the \autoref{Eq. Algorithm diag gain update}, we will have:
$\mK_k(\mR_k +\mH_k \mSigma_{k|k-1} \mH_k^{\top}) = \mSigma_{k|k-1} \mH_k^{\top} $.
So,
\begin{equation*}
\mK_k \mR_k = \mSigma_{k|k-1} \mH_k^{\top} - \mK_k \mH_k \mSigma_{k|k-1} \mH_k^{\top} = \big( \mSigma_{k|k-1} - \mK_k \mH_k \mSigma_{k|k-1} \big) \mH_k^{\top} .
\end{equation*}
By the standard formulation of covariance update, we know $\big( \mSigma_{k|k-1} - \mK_k \mH_k \mSigma_{k|k-1} \big) = \mSigma_{k}$, and thus: $\mK_k \mR_k = \mSigma_{k} \mH_k^{\top}$. Substituting this into the previous expression for $\vmu_k$, we obtain:
\begin{equation*}
\vmu_{k} = \vmu_{k|k-1} - \mSigma_{k} \mH_k^{\top} {\mH_k^{\top}}^{-1} \nabla_{\vtheta} \mathcal{L}_k .
\end{equation*}
Finally, simplifying yields $\vmu_{k} = \vmu_{k|k-1} - \mSigma_{k} \nabla_{\vtheta} \mathcal{L}_k$.
\end{proof}

\paragraph{Proof of Kalman–NGD Equivalence}
\label{app:fisher-kalman}

\begin{proof}
    We begin with the definition of the Fisher information matrix given in \autoref{eq:fisher}. We can rewrite it as: $\mF(\vtheta) = \sE_{p(\vy|\vx,\vtheta)}[ \nabla_{\vtheta} \mathcal{L} \cdot \nabla_{\vtheta} \mathcal{L}^{\top} ]$. At time step $k$, the new data point contributes statistical information in the form $\mF(\vtheta_k) = \nabla_{\vtheta} \mathcal{L}_k \cdot \nabla_{\vtheta} \mathcal{L}_k^{\top}$ to Fisher information matrix. Applying the chain rule, we write $\nabla_{\vtheta}\mathcal{L}_k = \nabla_{\hat{\vy}}\mathcal{L}_k \cdot \nabla_{\vtheta}\hat{\vy}_k$, which gives:
    \begin{equation*}
    \mF(\vtheta_k) = (\nabla_{\hat{\vy}}\mathcal{L}_k \cdot \nabla_{\vtheta}\hat{\vy}_k) \cdot (\nabla_{\hat{\vy}}\mathcal{L}_k \cdot \nabla_{\vtheta}\hat{\vy}_k)^{\top}.
    \end{equation*}
    Using matrix multiplication rules: $\mF(\vtheta_k) = \nabla_{\vtheta}^{\top}\hat{\vy}_k \cdot \nabla_{\hat{\vy}}^2 \mathcal{L}_k \cdot \nabla_{\vtheta}\hat{\vy}_k$. We already know that $\mH_k = \nabla_{\vtheta}\hat{\vy}_k$, and based on assumption of Gaussian (or more broadly, for an exponential family) likelihood, we obtain $\mR_k^{-1} = \nabla_{\hat{\vy}}^2 \mathcal{L}_k$. Consequently, $\mF(\vtheta_k) = \mH_k^{\top} \mR_k^{-1} \mH_k$, where corresponds to the update term in the information filtering covariance update in \autoref{Eq. Algorithm diag cov update}, i.e. $\mF \equiv \mSigma^{-1}$. Therefore, considering Lemma \ref{lemma:modified_kamlan_update}, the update direction in NGD and Kalman filter is equivalent to each other: $\mF(\vtheta_k)^{-1} \nabla_{\vtheta} \mathcal{L}_k \equiv \mSigma_{k} \nabla_{\vtheta} \mathcal{L}_k $.
\end{proof}

\section{Computational Complexity}
\label{Computational Complexity}

To analyze the computational complexity of the proposed algorithms, let $n$ denote the number of trainable parameters and $d_o$ the number of model outputs. Assume that at each iteration, a minibatch of size $m$ is received, for a total of $T$ iterations. Consider a model composed of $L$ layers, where the largest weight matrix has a size of at most $\omega \times \omega$. The computational complexity of the proposed algorithms is detailed below:

\subsection{Frequentist Computational Complexity}
\label{Frequentist Computational Complexity}

Since a forward pass through the model consists of exactly $L$ matrix multiplications, we define $C_1$ as the proportionality constant associated with each matrix multiplication. Similarly, $C_2$ denotes the constant for the backward pass, and $C_3$ corresponds to the computational cost of LU decomposition, as used in PyTorch's \verb|torch.inverse()| function. For the RING and RENG algorithms, let $S$ denote the skip frequency for computing the inverse Fisher matrix, and $K$ represent the number of iterations required for Newton’s method to approximate the inverse.

One iteration of the vanilla natural gradient descent method \citep{martens2015optimizing} involves the following computational steps: a forward pass with complexity ($C_1 L \omega^2 m$), a backward pass costing ($C_2 L \omega^2 m$), and an output sampling step with cost $\mathcal{O}(1)$. An additional backward pass is required to compute the error for each layer, incurring a further cost of ($C_2 L \omega^2 m$). Moreover, three matrix multiplications per layer are performed to compute the activations, error matrices, and batch gradients, resulting in a cost of ($3 C_1 L \omega^2 m$). Each layer also requires two matrix inversions, with a total cost of ($2 C_3 L \omega^3$) and two additional matrix multiplications for computing the final weight update, contributing ($2 C_1 L \omega^2$) to the overall complexity.

For the RING and RENG algorithms, the inverse Fisher matrix is computed once every $S$ steps, resulting in a total of $\frac{T}{S}$ exact computations throughout the training process. During the remaining $ T - \frac{T}{S}$ iterations, an approximate inverse is obtained using the \textit{Lazy Fisher} technique, as described in \autoref{eq:inverse_update}. In RENG, for each of the $\frac{T}{S}$ iterations involving exact inverse computation, the following operations are performed: One forward pass with complexity ($C_1 L \omega^2 m$); two backward passes (double backpropagation) costing ($2 C_2 L \omega^2 m$); one output sampling step with cost ($\mathcal{O}(1)$); an additional backward pass to compute layer-wise error: $C_2 L \omega^2 m$; three matrix multiplications per layer to compute activations, error matrices, and batch gradients: ($3 C_1 L \omega^2 m$); two approximate matrix inversions per layer using Newton’s method: ($2 \psi C_1 L \omega^2 K$); and two final matrix multiplications per layer for the weight update: ($2 C_1 L \omega^2$). Here, $\psi \in \{2,3,4\}$ denotes the number of matrix multiplications required per iteration of Newton’s method, depending on the chosen approximation: quadratic \autoref{eq:quad_newt}, cubic \autoref{eq:cube_newt}, or quartic \autoref{eq:quart_newt}:

Quadratic Approximation:
\begin{align}
  \mX_{k+1} = \mX_k (2 \mI - \mA \mX_k),
  \label{eq:quad_newt}
\end{align}
Cubic Approximation:
\begin{align}
  \mX_{k+1} = \mX_k (3 \mI - 3 \mA \mX_k + (\mA \mX_k)^2),
  \label{eq:cube_newt}
\end{align}
Quartic Approximation:
\begin{align}
  \mX_{k+1} = \mX_k (4 \mI - \mA \mX_k (6 \mI - \mA \mX_k (4 \mI - \mA \mX_k))).
  \label{eq:quart_newt}
\end{align}

Thus, the total time complexity of vanilla natural gradient descent over $T$ iterations is given by $T \times (C_1 L \omega^2 m + C_2 L \omega^2 m + C_2 L \omega^2 m + 3 C_1 L \omega^2 m + 2 C_3 L \omega^3 + 2 C_1 L \omega^2)$ which is of the order $\mathcal{O}(T(C L \omega^2 m + C' L \omega^3))$.

For RENG, the total complexity is given by: $\frac{T}{S} \times (C_1 L \omega^2 m + 2 C_2 L \omega^2 m + C_2 L \omega^2 m + 3 C_1 L \omega^2 m + 2 \psi C_1 L \omega^2 K + 2 C_1 L \omega^2) + \left( T - \frac{T}{S}\right) \times (C_1 L \omega^2 m + 2 C_1 L \omega^2 + 2 C_1 L \omega^2)$ which is of the order $\mathcal{O}\left( \frac{T}{S}(C_1 L \omega^2 m + C_1' L \omega^2 K)\right)$.

For RING, the total complexity is $\frac{T}{S} \times (C_1 L \omega^2 m + C_2 L \omega^2 m + C_2 L \omega^2 m + 3 C_1 L \omega^2 m + 2 \psi C_1 L \omega^2 K + 2 C_1 L \omega^2) + \left( T - \frac{T}{S}\right) \times (C_1 L \omega^2 m + 2 C_1 L \omega^2 + 2 C_1 L \omega^2)$ which is of the order $\mathcal{O}\left( \frac{T}{S}(C_2 L \omega^2 m + C_2' L \omega^2 K)\right)$.

Compared to vanilla natural gradient descent, our proposed methods achieve significant computational savings. By skipping the full Fisher inverse computation every iteration, the leading factor of $T$ in the complexity is effectively reduced to $\frac{T}{S}$. Additionally, replacing the costly matrix inversion (with complexity $\omega^3$) with Newton's iterative method reduces the cost to $\omega^2K$, where $K << \omega$. In practice, this iterative approach scales efficiently for large matrices, especially under low-precision arithmetic (e.g., \verb|bFloat16|). Furthermore, modern GPU architectures enable highly parallel matrix multiplications, with row- and column-wise operations executed concurrently, which makes Newton's method substantially more scalable than traditional LU decomposition-based inversion.

\subsection{Bayesian Computational Complexity}
\label{Bayesian Computational Complexity}

We define $C_1$ as the proportionality constant associated with the cost of matrix multiplication, $C_2$ for Jacobian computation, and $C_3$ for inverting the observation noise covariance matrix. We analyze the computational complexity of both the standard Kalman filter and our proposed Jacobian-regularized Kalman algorithm with diagonal approximation, as introduced by \citep{chang2022diagonal}.
Our empirical observations also demonstrate that throughout the training process, the covariance matrix asymptotically approaches a (block-)diagonal configuration: $\sE_{p^*(\vx,\vy)}[\vsigma] = \text{diag}(\mSigma)$. 
% 
% \autoref{Fig:diag} illustrates the evolution of the covariance matrix $\mSigma_{k} \in \sR^{n \times n}$.
% Initially, the matrix exhibits a fully dense positive-definite form, progressively transitioning towards a (block-)diagonal configuration as the training algorithm advances.
% 
% \begin{figure}[h]
%   \centering
%   \begin{minipage}{0.3\linewidth}
%       \centering
%       \includegraphics[width=\linewidth]{Figure/Other/diag-a.png}
%       % \caption{Initial P}
%       \label{Fig:diag-a}
%   \end{minipage}
%   \hfill
%   \begin{minipage}{0.3\linewidth}
%       \centering
%       \includegraphics[width=\linewidth]{Figure/Other/diag-c.png}
%       % \caption{Initial P}
%       \label{Fig:diag-c}
%   \end{minipage}
%   \hfill
%   \begin{minipage}{0.3\linewidth}
%       \centering
%       \includegraphics[width=\linewidth]{Figure/Other/diag-b.png}
%       % \caption{After training P}
%       \label{Fig:diag-b}
%   \end{minipage}
%   \caption{Evolution of covariance matrix $\mSigma_{k} \in \sR^{n \times n}$. The matrix starts with a fully dense positive-definite matrix, and with the progress of the training algorithm, it gradually converges to a (block-)diagonal configuration.}
%   \label{Fig:diag}
% \end{figure}
% 
To implement the diagonal approximation in the regularized Kalman algorithm, we use the following update step:
\begin{subequations} \label{Eq. LoKO Updating}
    \begin{align} 
        \tilde{\mK}_k &= \vsigma_{k|k-1} \bullet \mH_k^{\top} \left(\mH_k (\vsigma_{k|k-1} \bullet \mH_k^{\top}) + \mR_k (\mI + \rho \mR_k)^{-1} \right)^{-1}
        \label{Eq. diag gain update}\\
        \vmu_{k} &= \vmu_{k|k-1} + \tilde{\mK}_k (\vy_k - h(\vx_{k}, \vmu_{k|k-1}))
        \label{Eq. Updating Process}\\
        \left(\vsigma_{k}\right)^i &= \left(\vsigma_{k|k-1}\right)^i - \left(\mK_k\right)^{i}_{j} \left(\mH_k\right)^{j}_{i} \left(\vsigma_{k|k-1}\right)^i
        \label{Eq. diag cov update}
    \end{align}
\end{subequations}
where the symbol $\bullet$ represents the transposed Khatri–Rao product, which is essentially the row-by-row Kronecker product of the vector $\vsigma_{k|k-1}$ and matrix $\mH_k^{\top}$. The \autoref{Eq. diag cov update} represents the diagonal update of the covariance matrix, expressed with Einstein notation.
Note that the operations used in \autoref{Eq. diag gain update}, and \autoref{Eq. diag cov update} can be computed efficiently. For example, in PyTorch, they can be seamlessly implemented using the $*$, and $\textbf{einsum}()$ operators, streamlining the computational process.
The analysis is structured around the three main steps of the algorithm: Prediction, Pre-Update, and Update for both standard and our Kalman algorithms.

The computational complexity of both our algorithm and the standard Kalman filter for the prediction step is $\mathcal{O}(1)$.

The pre-update step for Jacobian computation incurs a cost of $C_2 n d_{o}$. Estimating the observation noise covariance involves a cost of $C_1 (d_{o}^2 + n d_{o}^2 + n^2 d_{o})$, while updating the observation noise covariance via exponential averaging has negligible cost, i.e., $\mathcal{O}(1)$. Computing the Kalman gain requires two matrix multiplications and one matrix inversion (of size $d_{o} \times d_o$), with a total cost of $(C_1 n^2 d_{o} + C_3 d_{o}^3 + C_1 n d_{o}^2)$. The parameter update, which involves multiplying the Kalman gain matrix by the output error, costs $C_1 n d_{o}$. The final and most computationally intensive step is the update of the posterior covariance matrix $\mSigma_k$, which requires constructing a full $n \times n$ matrix. This step contributes $(C_1 n^2 d_{o} + C_1 n^3)$ to the overall cost. Therefore, the total time complexity across $T$ iterations is: $T \times (C_2 n d_{o} + C_1 (d_{o}^2 + n d_{o}^2 + n^2 d_{o}) + C_1 n^2 d_{o} + C_3 d_{o}^3 + C_1 n d_{o}^2 + C_1 n d_{o} + C_1 n^2 d_{o} + C_1 n^3)$. Since $d_{o} \ll n$, the dominant terms are $C_1n^2 d_{o}$ and $C_1n^3$, which yields an overall computational complexity of order $\mathcal{O}(C_1 (n^2 d_{o} + n^3))$.

In the case where we employ the diagonal approximation of the covariance matrix, the use of the element-wise Khatri–Rao product replaces matrix–matrix multiplications with more efficient matrix–vector operations. Consequently, the Kalman gain computation now incurs a cost of $2 (C_1 n + C_1 n d_{o}^2 + C_3 d_{o}^3)$, as it includes an additional matrix inversion due to the newly introduced regularization. The parameter update step maintains the same cost as in the full covariance case. 
For the covariance update, which is expressed using Einstein summation notation, each element of the diagonal covariance requires a vector–vector multiplication. Thus, updating all $n$ diagonal elements costs $C_1 n^2$. As a result, the total computational complexity per iteration becomes $\mathcal{O}(C_1 (n^2 + n d_{o}^2))$, which is one order of magnitude lower than that of the standard Kalman update.

\clearpage

\section{Hyperparameters}
\label{Hyperparameters}

The hyperparameters used in vision and language experiments are reported in \autoref{Table: Hyperparameters Vision} and \autoref{Table: Hyperparameters Language}, respectively.

\begin{table}[h]
  \caption{Hyperparameters used for fine-tuning ViT-B16 on various image classification datasets, including CIFAR-10, CIFAR-100, Oxford-IIIT Pet, Food-101, and ImageNet-100. The table lists the hyperparameter settings for our proposed optimization algorithms (RING, RENG, and R-Kalman), as well as for the baseline methods AdamW, Sophia, and NGD.}
  \label{Table: Hyperparameters Vision}
  \centering
  \scriptsize{
  \begin{tabular}{ll|cccccccc}
    \toprule
    \textbf{Algorithm} & \textbf{Hyperparameters} & \textbf{CIFAR-10}  &  \textbf{CIFAR-100} & \textbf{Oxford-IIIT Pet} & \textbf{Food-101} & \textbf{ImageNet-100} \\
    \midrule
    \multirow{6}{*}{\textbf{AdamW}} & Batch Size & \multicolumn{5}{c}{16} \\
            & Num. Epochs & $10$ & $5$ & $10$ & $8$ & $8$ \\
            & Learning Rate $(\times 10^{-4})$ & $0.6$-$0.8$ & $0.4$-$0.6$ & $0.6$-$0.8$ & $0.4$-$0.8$ & $0.6$-$0.8$ \\
            & Warmup Steps & $200$ & $500$ & $500$ & $500$ & $500$ \\
            & LoRA Config. & \multicolumn{5}{c}{$r_q=r_v=4$} \\
            & LoRA $\alpha$ & \multicolumn{5}{c}{$4$} \\ \addlinespace[1ex]
    \midrule
    \multirow{7}{*}{\textbf{Sophia}} & Batch Size & \multicolumn{5}{c}{16} \\
            & Num. Epochs & $10$ & $5$ & $10$ & $8$ & $8$ \\
            & Learning Rate $(\times 10^{-4})$ & $0.6$-$0.8$ & $0.4$-$0.6$ & $0.6$-$0.8$ & $0.4$-$0.8$ & $0.6$-$0.8$ \\
            & Warmup Steps & $200$ & $500$ & $500$ & $500$ & $500$ \\
            & $\beta_1, \beta_2$ & \multicolumn{5}{c}{$0.975, 0.99$} \\
            & LoRA Config. & \multicolumn{5}{c}{$r_q=r_v=4$} \\
            & LoRA $\alpha$ & \multicolumn{5}{c}{$4$} \\
    \addlinespace[1ex]
    \midrule
    \multirow{7}{*}{\textbf{NGD}} & Batch Size & \multicolumn{5}{c}{100} \\
            & Num. Epochs & $3$ & $2$ & $3$ & $2$ & $2$ \\
            & Learning Rate & \multicolumn{5}{c}{$1.0$} \\
            & $\rho$ $(\times 10^{-4})$ & $1.0$ & $1$ & $2$ & $2$ & $0.8$  \\
            & $\phi$ & \multicolumn{5}{c}{$0.995$} \\
            & LoRA Config. & \multicolumn{5}{c}{$r_q=r_v=4$} \\
            & LoRA $\alpha$ & \multicolumn{5}{c}{$4$} \\
    \addlinespace[1ex]
    \midrule
    \multirow{9}{*}{\textbf{RING}} & Batch Size & \multicolumn{5}{c}{100} \\
            & Num. Epochs & $4$ & $1$ & $2$ & $2$ & $2$ \\
            & Learning Rate & \multicolumn{5}{c}{$1.0$} \\
            & $\rho$ $(\times 10^{-4})$ & $10$ & $10$ & $4$ & $5$ & $8$ \\
            & $\phi$ & \multicolumn{5}{c}{$0.99$-$0.995$} \\
            & Skip freq. $(S)$ & $8$ & $8$ & $8$ & $8$ & $8$ \\
            & LoRA Config. & \multicolumn{5}{c}{$r_q=r_v=4$} \\
            & LoRA $\alpha$ & \multicolumn{5}{c}{$4$} \\
    \addlinespace[1ex]
    \midrule
    \multirow{10}{*}{\textbf{RENG}} & Batch Size & \multicolumn{5}{c}{100} \\
            & Num. Epochs & $4$ & $2$ & $3$ & $2$ & $2$ \\
            & Learning Rate & \multicolumn{5}{c}{$1.0$} \\
            & Grad Reg. & $0.05$ & $0.05$ & $0.01$ & $0.01$ & $0.01$ \\
            & $\rho$ $(\times 10^{-4})$ & $40$ & $10$ & $8$ & $2$ & $8$ \\
            & $\phi$ & \multicolumn{5}{c}{$0.99$-$0.995$} \\
            & Skip freq $(S)$ & $8$ & $8$ & $8$ & $8$ & $8$ \\
            & LoRA Config. & \multicolumn{5}{c}{$r_q=r_v=4$} \\
            & LoRA $\alpha$ & \multicolumn{5}{c}{$4$} \\
    \addlinespace[1ex]
    \midrule
    \multirow{7}{*}{\textbf{R-Kalman}} & Batch Size & \multicolumn{5}{c}{1} \\
            & Num. Epochs & $1$ & $1$ & $2$ & $1$ & $1$ \\
            & $\beta$ & $0.96$ & $0.98$ & $0.96$ & $0.98$ & $0.98$ \\
            & $\vsigma_0$ & $0.2$ & $0.2$ & $0.2$ & $0.1$ & $0.1$ \\
            & LoRA Config. & \multicolumn{5}{c}{$r_q=r_v=4$} \\
            & LoRA $\alpha$ & \multicolumn{5}{c}{$4$} \\
    \bottomrule
  \end{tabular}
    {\scriptsize 
    
    Discount factor $\phi$ is used in the Levenberg-Marquardt scheme for updating the damping coefficient \citep{martens2015optimizing}.}
  }
\end{table}

\clearpage

\begin{table}[H]
  \caption{Hyperparameters used for fine-tuning RoBERTa-base on the GLUE language benchmarks, including MNLI-mm, QQP, QNLI, SST-2, CoLA, STS-B, MRPC, and RTE. The table reports the hyperparameter settings for our proposed optimization algorithms (RING, RENG, and R-Kalman) as well as for the baseline methods AdamW, Sophia, and NGD.}
  \label{Table: Hyperparameters Language}
  \centering
  \scriptsize{
  \begin{tabular}{l@{\hskip 1em}p{0.13\linewidth}|cccccccc}
    \toprule
    \textbf{Algorithm} & \textbf{Hyperparameters} & \textbf{MNLI-mm}  &  \textbf{QQP} & \textbf{QNLI} & \textbf{SST-2} & \textbf{CoLA} & \textbf{STS-B} & \textbf{MRPC} & \textbf{RTE} \\
    \midrule
    \multirow{6}{*}{\textbf{AdamW}} & Batch Size & \multicolumn{8}{c}{64} \\
            & Num. Epochs & $8$ & $8$ & $8$ & $10$ & $25$ & $40$ & $40$ & $25$ \\
            & Learning Rate $(\times 10^{-4})$ & $0.4$-$0.8$ & $0.6$-$0.8$ & $1.0$-$2.0$ & $0.8$-$1.2$ & $0.4$-$0.6$ & $0.4$-$1.0$ & $0.8$-$1.0$ & $0.8$-$1.0$ \\
            & Warmup Steps & $500$ & $500$ & $500$ & $500$ & $200$ & $100$ & $200$ & $100$ \\
            & LoRA Config. & \multicolumn{8}{c}{$r_q=r_v=8$} \\
            & LoRA $\alpha$ & \multicolumn{8}{c}{8} \\
            & Max. Seq. Len. & \multicolumn{8}{c}{$512$} \\ \addlinespace[1ex]
    \midrule
    \multirow{7}{*}{\textbf{Sophia}} & Batch Size & \multicolumn{8}{c}{64} \\
            & Num. Epochs & $8$ & $8$ & $8$ & $10$ & $25$ & $40$ & $40$ & $25$ \\
           & Learning Rate $(\times 10^{-4})$ & $0.4$-$0.8$ & $0.6$-$0.8$ & $2.0$-$4.0$ & $0.8$-$1.2$ & $0.6$-$0.8$ & $0.4$-$1.0$ & $0.8$-$1.0$ & $0.8$-$1.0$ \\
            & $\beta_1, \beta_2$ & \multicolumn{8}{c}{$0.965, 0.98$} \\
            & LoRA Config. & \multicolumn{8}{c}{$r_q=r_v=8$} \\
            & LoRA $\alpha$ & \multicolumn{8}{c}{$8$} \\
            & Max. Seq. Len. & \multicolumn{8}{c}{$512$} \\
    \addlinespace[1ex]
    \midrule
    \multirow{7}{*}{\textbf{NGD}} & Batch Size & \multicolumn{8}{c}{64} \\
            & Num. Epochs & $1$ & $1$ & $2$ & $3$ & $6$ & $10$ & $5$ & $5$ \\
            & Learning Rate & \multicolumn{8}{c}{$0.99$-$1.0$} \\
            & $\rho$ $(\times 10^{-4})$ & $1.0$ & $1$ & $2$ & $2$ & $0.8$ & $1.0$ & $0.8$ & $0.4$ \\
            & $\phi$ & \multicolumn{8}{c}{$0.99$-$0.995$} \\
            & LoRA Config. & \multicolumn{8}{c}{$r_q=r_v=8$} \\
            & LoRA $\alpha$ & \multicolumn{8}{c}{$8$} \\
            & Max. Seq. Len. & \multicolumn{8}{c}{$512$} \\
    \addlinespace[1ex]
    \midrule
    \multirow{9}{*}{\textbf{RING}} & Batch Size & $64$ & $64$ & $64$ & $32$ & $64$ & $100$ & $32$ & $32$ \\
            & Num. Epochs & $1$ & $1$ & $2$ & $5$ & $6$ & $10$ & $5$ & $5$ \\
            & Learning Rate & \multicolumn{8}{c}{$0.9$-$1.0$} \\
            & $\rho$ $(\times 10^{-4})$ & $10$ & $10$ & $4$ & $5$ & $8$ & $40$ & $9$ & $10$ \\
            & $\phi$ & \multicolumn{8}{c}{$0.999$-$0.9999$} \\
            & Skip freq $(S)$ & $8$ & $8$ & $8$ & $4$ & $4$ & $4$ & $4$ & $4$ \\
            & LoRA Config. & \multicolumn{8}{c}{$r_q=r_v=8$} \\
            & LoRA $\alpha$ & \multicolumn{8}{c}{$8$} \\
            & Max. Seq. Len. & \multicolumn{8}{c}{$512$} \\
    \addlinespace[1ex]
    \midrule
    \multirow{10}{*}{\textbf{RENG}} & Batch Size & $64$ & $64$ & $64$ & $256$ & $64$ & $100$ & $32$ & $64$ \\
            & Num. Epochs & $1$ & $1$ & $2$ & $5$ & $6$ & $10$ & $5$ & $5$ \\
            & Learning Rate & \multicolumn{8}{c}{$0.9$-$1.0$} \\
            & Grad. regu. & $0.005$ & $0.01$ & $0.01$ & $0.001$ & $0.1$ & $0.001$ & $0.001$ & $0.001$ \\
            & $\rho$ $(\times 10^{-4})$ & $40$ & $10$ & $8$ & $2$ & $8$ & $40$ & $8$ & $10$ \\
            & $\phi$ & \multicolumn{8}{c}{$0.999$-$0.9999$} \\
            & Skip freq $(S)$ & $8$ & $8$ & $8$ & $4$ & $8$ & $4$ & $4$ & $4$ \\
            & LoRA Config. & \multicolumn{8}{c}{$r_q=r_v=8$} \\
            & LoRA $\alpha$ & \multicolumn{8}{c}{$8$} \\
            & Max. Seq. Len. & \multicolumn{8}{c}{$512$} \\
    \addlinespace[1ex]
    \midrule
    \multirow{7}{*}{\textbf{R-Kalman}} & Batch Size & \multicolumn{8}{c}{1} \\
            & Num. Epochs & $1$ & $1$ & $1$ & $2$ & $2$ & $2$ & $4$ & $4$ \\
            & $\beta$ & $0.97$ & $0.97$ & $0.97$ & $0.96$ & $0.98$ & $0.98$ & $0.98$ & $0.98$ \\
            & $\vsigma_0$ & $0.1$ & $0.1$ & $0.1$ & $0.1$ & $0.04$ & $0.05$ & $0.05$ & $0.05$ \\
            & LoRA Config. & \multicolumn{8}{c}{$r_q=r_v=8$} \\
            & LoRA $\alpha$ & \multicolumn{8}{c}{$8$} \\
            & Max. Seq. Len. & \multicolumn{8}{c}{$512$} \\
    \bottomrule
  \end{tabular}
  {\scriptsize 
  
    Discount factor $\phi$ is used in the Levenberg-Marquardt scheme for updating the damping coefficient \citep{martens2015optimizing}.}
  }
\end{table}

\clearpage

\section{Additional Experiments}
\label{sec:Additional Experiments}

We perform extensive empirical evaluations across both image classification and GLUE language benchmarks. Detailed results for vision and language tasks are presented in \autoref{fig:image_results} and \autoref{fig:nlp_results}, respectively. The vision datasets include CIFAR-10/100, Oxford-IIIT Pet, Food-101, and ImageNet-100, while the GLUE benchmarks cover MNLI-mm, QQP, QNLI, SST-2, CoLA, STS-B, MRPC, and RTE subsets. For image classification, we fine-tune a ViT-B16 model pre-trained on ImageNet; for language tasks, we fine-tune RoBERTa-base, pre-trained on large-scale text corpora. The plots show the validation accuracy of our proposed algorithms (RING, RENG, and R-Kalman) compared to AdamW, Sophia, and NGD as a function of training iterations.

\begin{figure}[H]
\begin{flushleft}
    \includegraphics[width=\textwidth]{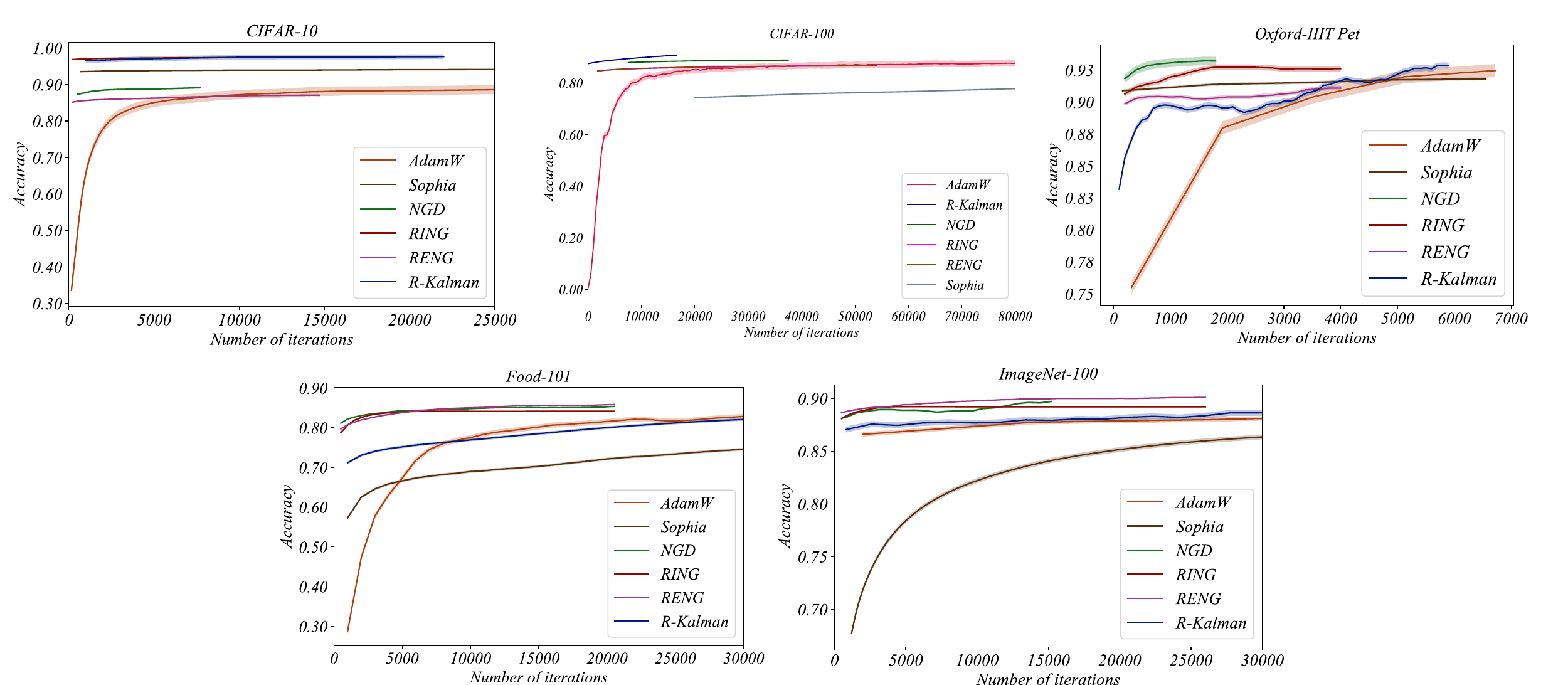}
    \caption{Comprehensive results on the image classification benchmarks, including CIFAR-10/100, Oxford-IIIT Pet, Food-101, and ImageNet-100. The plots show the validation accuracy of our proposed algorithms (RING, RENG, and R-Kalman) compared to AdamW, Sophia, and NGD when fine-tuning ViT-B16. Validation accuracy is plotted as a function of training iterations.}
    \label{fig:image_results}
\end{flushleft}
\end{figure}

\clearpage

\begin{figure}[H]
\begin{flushleft}
    \includegraphics[width=\textwidth]{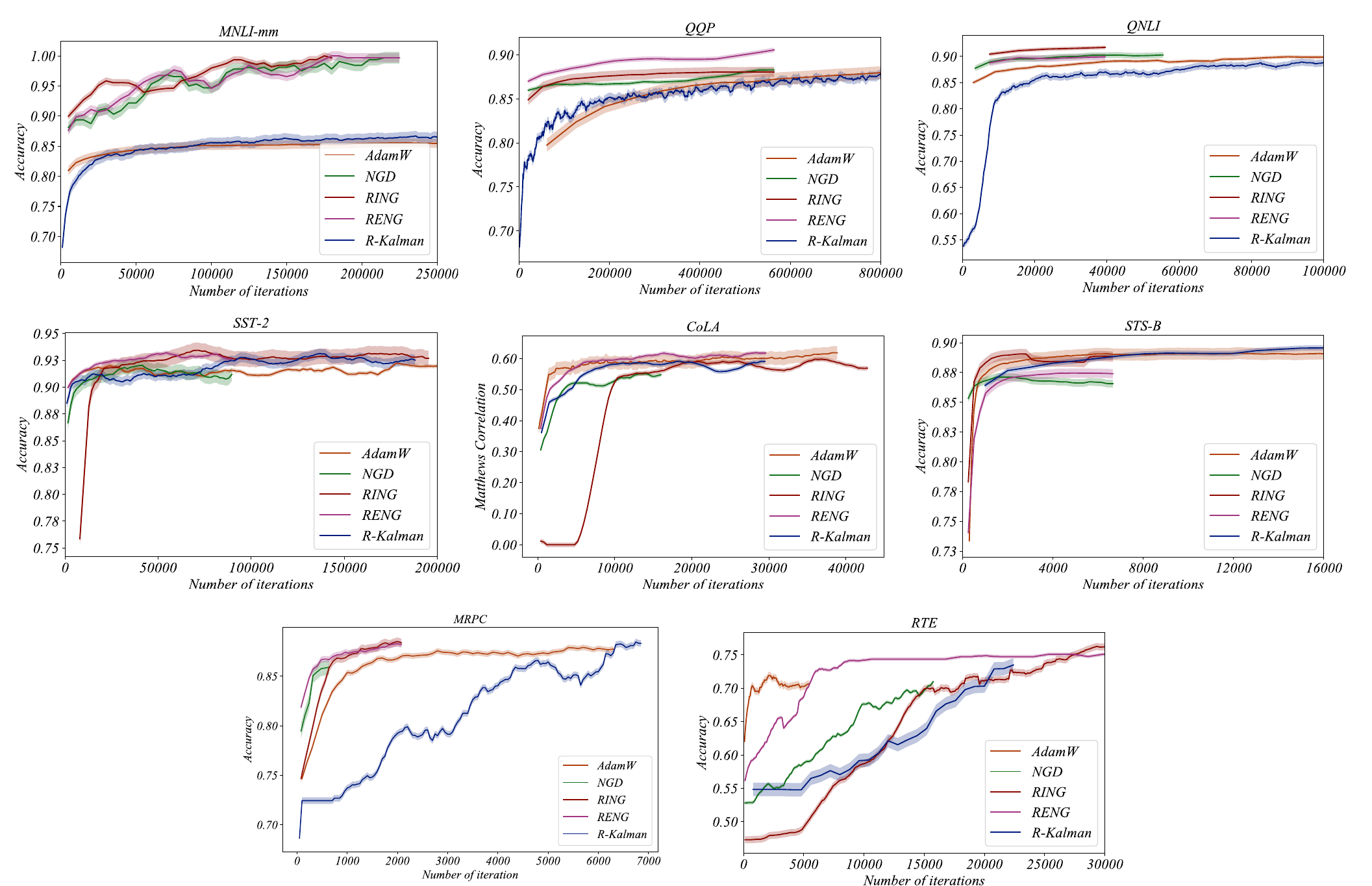}
    \caption{Comprehensive results on the GLUE benchmark, including MNLI-mm, QQP, QNLI, SST-2, CoLA, STS-B, MRPC, and RTE. The plots show the validation accuracy of our proposed algorithms (RING, RENG, and R-Kalman) compared to AdamW and NGD when fine-tuning RoBERTa-base. Validation accuracy is plotted as a function of training iterations.}
    \label{fig:nlp_results}
\end{flushleft}
\end{figure}

\clearpage

\section{Ablation Study}
\label{sec:Ablation Study}

\paragraph{Sensitivity to damping coefficient and gradient regularization coefficient.}
We demonstrate that, within a certain range, varying the gradient regularization coefficient in RING and RENG, as well as the damping coefficient in RENG, leads to convergence toward near-optimal solutions. However, our analysis also reveals a lower bound on these hyperparameters, which may depend on the dataset. It is important to note that, due to our use of the Lazy-Fisher technique, inaccurate estimates of the Fisher information matrix can lead to instability. In particular, setting the damping coefficient too low. For more details, see \autoref{fig:damping_ring} for RING algorithm, and \autoref{fig:Reg_abalations} for RENG.

\begin{figure}[h]
\begin{centering}
    \includegraphics[width=0.8\textwidth]{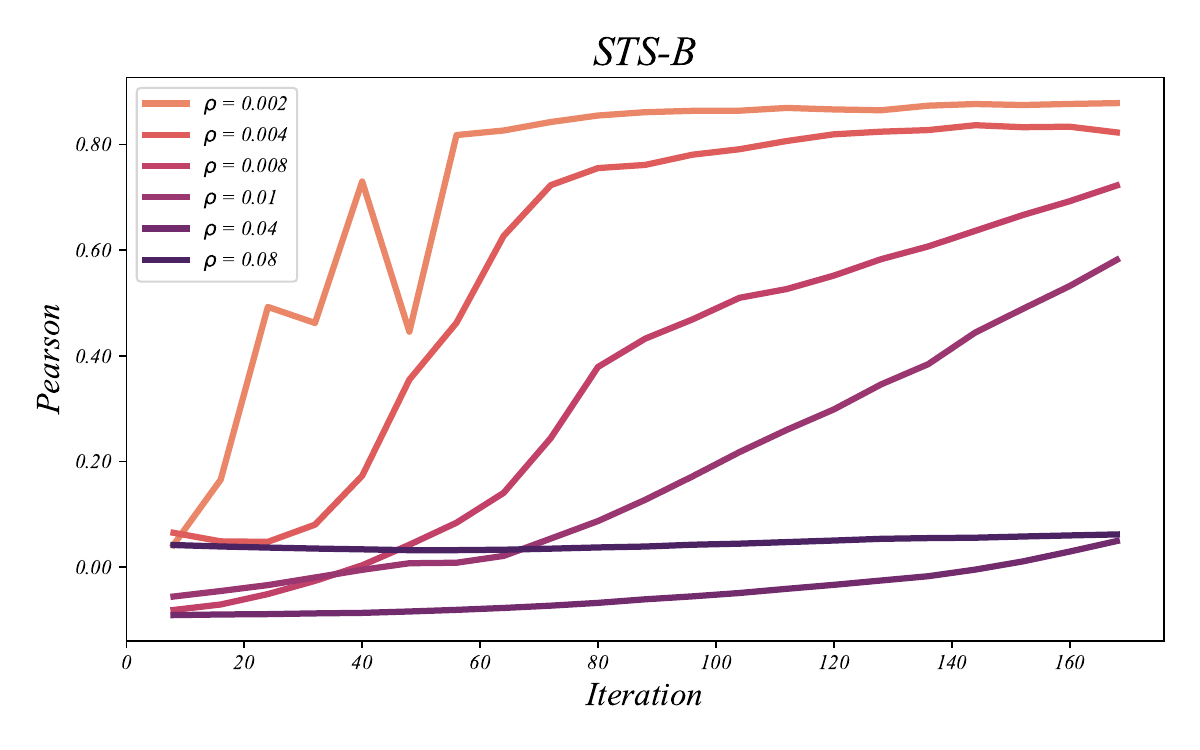}
    \caption{Sensitivity analysis of dampening coefficient for RING.}
    \label{fig:damping_ring}
\end{centering}
\end{figure}

\begin{figure}[h]
\centering
\begin{minipage}[b]{0.49\textwidth}
\centering
\includegraphics[width=\textwidth]{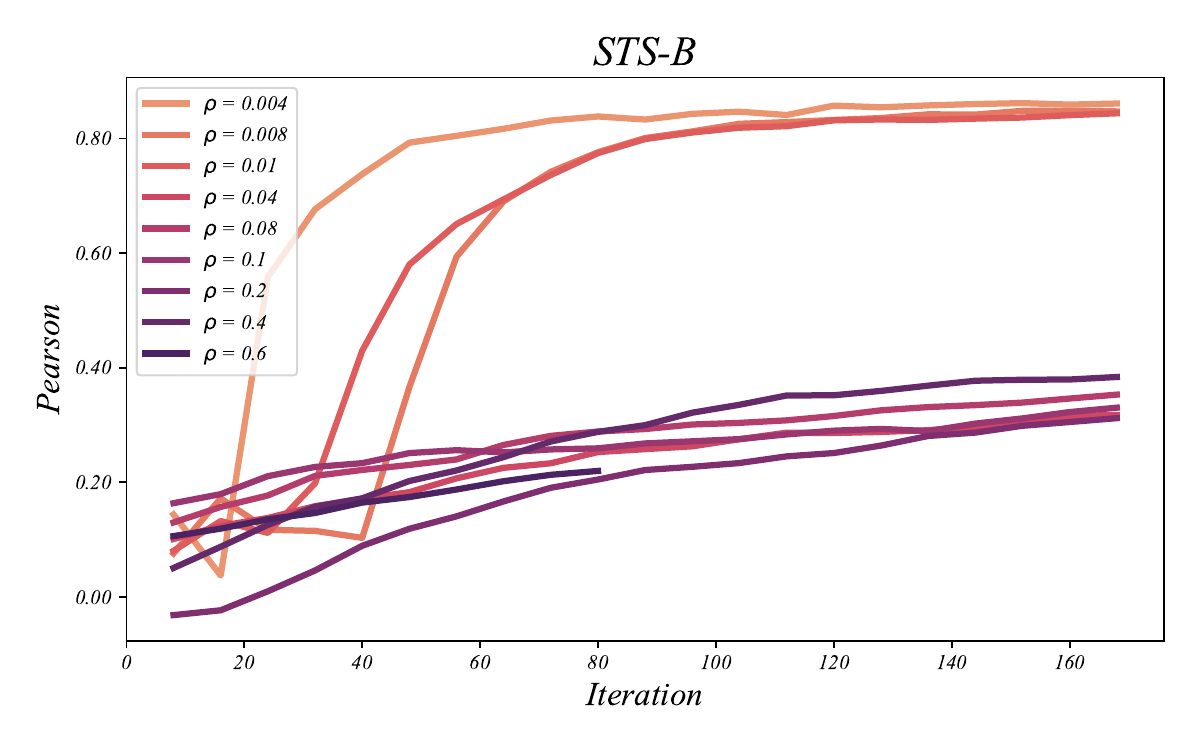}
% \caption{Image 1}
\label{fig:reng_damp_abalation}
\end{minipage}
\hfill
\begin{minipage}[b]{0.49\textwidth}
\centering
\includegraphics[width=\textwidth]{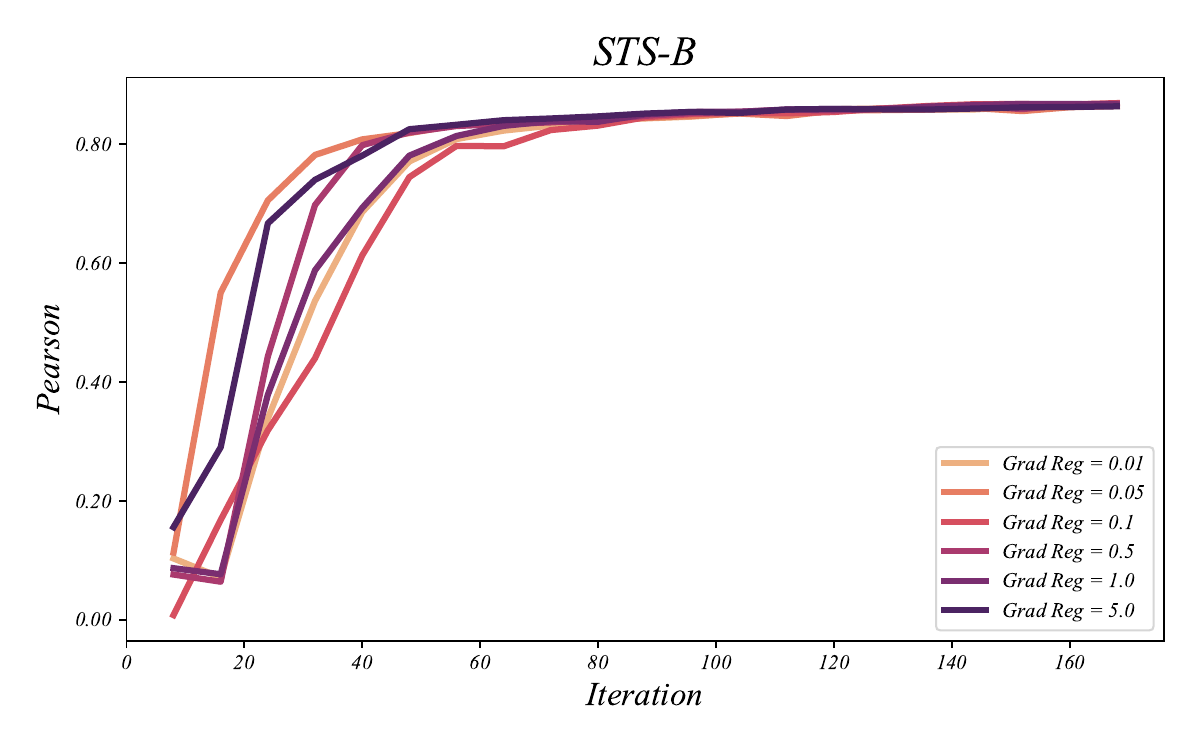}
% \caption{Image 2}
\label{fig:ring_damp_abalation}
\end{minipage}
\caption{Sensitivity analysis of dampening coefficient and gradient regularization coefficient for RENG.}
\label{fig:Reg_abalations}
\end{figure}

\paragraph{Skip Frequency in Lazy-Fisher.}
We systematically investigated the impact of update frequency in Lazy-Fisher on runtime and performance. The results of these experiments are presented in the \autoref{Table: skip_freq_lazy_fisher_1}, and \autoref{Table: skip_freq_lazy_fisher_2}. The first table presents results on the STS-B dataset, reporting values at iteration 20. The second table shows results on the MNLI-mm dataset, with values reported at iteration 200. The empirical results demonstrate that a skip frequency of 4 for small datasets and 8 for large datasets provides an effective balance between runtime efficiency and performance.

\begin{table}
  \small
  \caption{Performance and runtime on the STS-B dataset under different update frequencies of Lazy-Fisher.}
  \label{Table: skip_freq_lazy_fisher_1}
  \centering
  {\small
  \begin{tabular}{lccccc}
    \toprule
    \textbf{Algorithm} & \textbf{Skip Freq = 2} & \textbf{Skip Freq = 4}  &  \textbf{Skip Freq = 6} & \textbf{Skip Freq = 8} & \textbf{Skip Freq = 10} \\
    \midrule
    \multirow{2}{*}{\textbf{RING}} & Pearson: $35.85$  & Pearson: $79.54$ & Pearson: $65.23$ & Pearson: $58.57$ & Pearson: $69.24$  \\ 
    & Runtime: $2:46 s$  & Runtime: $1:40 s$ & Runtime: $1:20 s$ & Runtime: $1:07 s$ & Runtime: $1:00 s$ \\ \addlinespace[1ex]
    \multirow{2}{*}{\textbf{RENG}}      & Pearson: $40.23$ & Pearson: $52.92$ & Pearson: $48.12$ & Pearson: $40.12$ & Pearson: $33.15$ \\ 
    & Runtime: $1:32 s$ & Runtime: $0:51 s$ & Runtime: $0:43 s$ & Runtime: $0:36 s$ & Runtime: $0:30 s$ \\ 
    \bottomrule
  \end{tabular}
  }
\end{table}

\begin{table}
  \small
  \caption{Performance and runtime on the MNLI-mm dataset under different update frequencies of Lazy-Fisher.}
  \label{Table: skip_freq_lazy_fisher_2}
  \centering
  {\small
  \begin{tabular}{lcccc}
    \toprule
    \textbf{Algorithm} & \textbf{Skip Freq = 4} & \textbf{Skip Freq = 8}  &  \textbf{Skip Freq = 16} & \textbf{Skip Freq = 32} \\
    \midrule
    \multirow{2}{*}{\textbf{RING}} & Accuracy: $48.43 \%$  & Accuracy: $46.87 \%$ & Accuracy: $37.50 \%$ & Accuracy: $39.08 \%$  \\ 
    & Runtime: $31:27 s$  & Runtime: $21:10 s$ & Runtime: $16:45 s$ & Runtime: $13:11 s$ \\ \addlinespace[1ex]
    \multirow{2}{*}{\textbf{RENG}}      & Accuracy: $51.56 \%$ & Accuracy: $51.56 \%$ & Accuracy: $50.00 \%$ & Accuracy: $46.87 \%$ \\ 
    & Runtime: $8:00 s$ & Runtime: $5:15 s$ & Runtime: $4:00 s$ & Runtime: $3:15 s$ \\ 
    \bottomrule
  \end{tabular}
  }
\end{table}

\paragraph{Initial Prior $\mathN(\vmu_0,\mSigma_0)$.}
We find that selecting an appropriate initial prior $\mathN(\vmu_0,\mSigma_0)$ is crucial for ensuring stable convergence of the Kalman algorithm. This selection ideally relies on prior knowledge of the trainable parameters and their associated uncertainties. In the absence of such information, a practical strategy is to adopt widely used initialization schemes such as Xavier \citep{glorot2010understanding} or Kaiming \citep{he2015delving} for training scenarios. For fine-tuning settings, pre-trained parameter values may be used, or alternatively, standard initialization methods employed in parameter-efficient fine-tuning approaches such as LoRA \citep{hu2021lora}.

As for the covariance matrix $\mSigma_0$, it can be initialized with arbitrary positive values, as long as they are neither too small nor excessively large. A common and effective choice is $\mSigma_0 = \sigma_0\mI$, where $0 < \sigma_0 \ll 1$. Importantly, the performance of our algorithm is not highly sensitive to the specific choice of initial values. Regardless of the initialization method, the algorithm consistently converges to a stable performance level over time.
\autoref{fig:sigma0_results} shows the sensitivity of R-Kalman algorithm to different values of $\sigma_0$ in $0 < \sigma_0 \ll 1$.

\begin{figure}
\begin{centering}
    \includegraphics[width=0.8\textwidth]{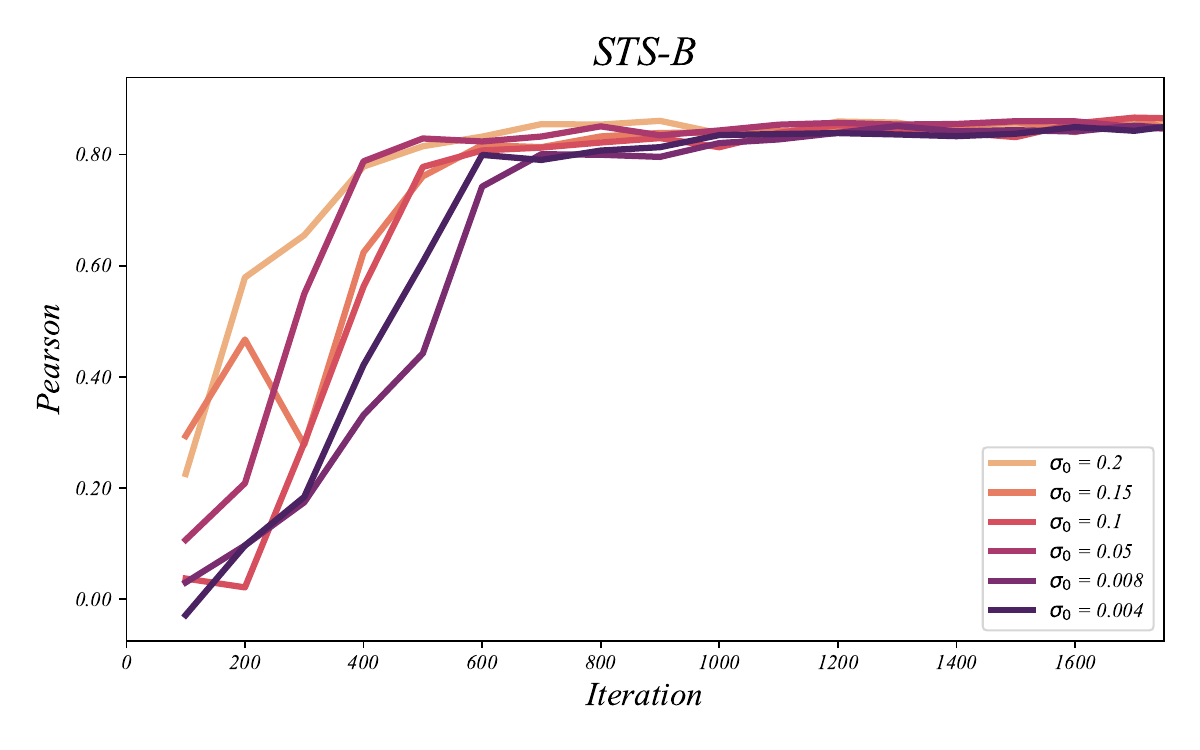}
    \caption{Sensitivity analysis of the R-Kalman algorithm with respect to the initial values of $\sigma_0$. The plot indicates that, as long as $\sigma_0$ lies within the $0 < \sigma_0 \ll 1$ range, the algorithm converges to a consistent optimal solution regardless of its specific initial value.}
    \label{fig:sigma0_results}
\end{centering}
\end{figure}

\paragraph{Sensitivity to $\beta$ in $\mR_k$ Estimation.}
We adopt the method proposed in \cite{abdi2024loko} to estimate the observation noise covariance matrix $\mR_k$:
\begin{align*}
        \mR_k &= \beta \mR_{k-1} + (1-\beta) \hat{\mR}_k , \\
        \hat{\mR}_k &= \left(\vy_k - h(\vx_{k}, \vmu_{k|k-1})\right)\left(\vy_k - h(\vx_{k}, \vmu_{k|k-1})\right)^{\top} + \mH_k \mSigma_{k|k-1} \mH_k^{\top} ,
\end{align*}
where $\beta \in (0,1)$ is forgetting factor.
To evaluate the influence of the hyperparameter \(\beta\) on the performance of the R-Kalman algorithm, we performed a sensitivity analysis by varying \(\beta\) across a range of values. The results demonstrate that setting \(\beta\) too high or too low significantly degrades the algorithm's performance. In contrast, optimal performance is consistently achieved when \(\beta\) lies within the range \(0.94 \leq \beta \leq 0.99\). Moreover, within this interval, the performance is relatively insensitive to the precise choice of \(\beta\). For a detailed comparison, refer to \autoref{fig:beta_results}.

\begin{figure}
\begin{centering}
    \includegraphics[width=0.8\textwidth]{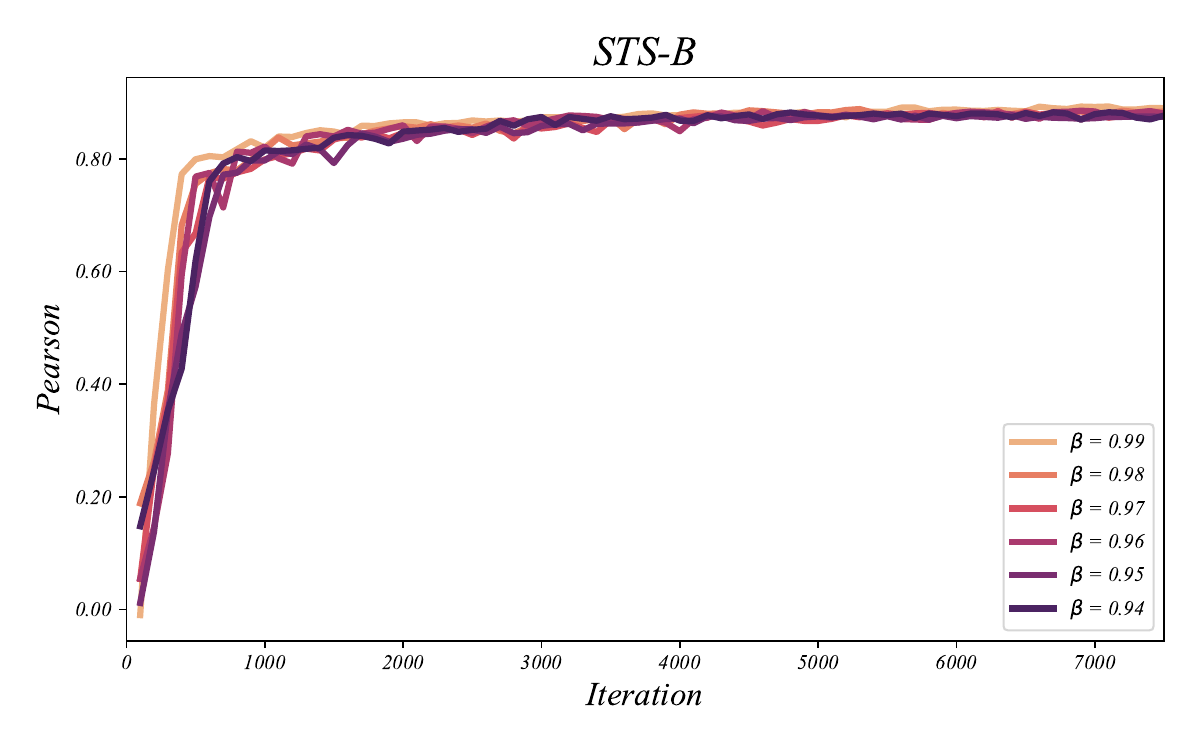}
    \caption{Sensitivity analysis of R-Kalman algorithm to $\beta$ values. The plot shows that, as long as $\beta$ lies within the range \(0.94 \leq \beta \leq 0.99\), the algorithm converges to a consistent optimal solution.}
    \label{fig:beta_results}
\end{centering}
\end{figure}

\newpage
\newpage

\clearpage

\end{document}